\documentclass{article}

\usepackage[preprint]{neurips_2026}

\usepackage[utf8]{inputenc} 
\usepackage[T1]{fontenc}    
\usepackage{hyperref}       
\usepackage{url}            
\usepackage{booktabs}       
\usepackage{amsfonts}       
\usepackage{nicefrac}       
\usepackage{microtype}      
\usepackage[table,dvipsnames]{xcolor}         

\usepackage{mathrsfs}
\usepackage{hyperref}
\usepackage{url}
\usepackage{amsthm}
\usepackage{algorithm}
\usepackage{algpseudocode}

\usepackage{graphicx}  

\usepackage{booktabs}  

\usepackage{multirow}  
\usepackage{array}     
\usepackage{pifont}    
\usepackage{caption}   
\usepackage{adjustbox}
\usepackage{subcaption}
\usepackage{makecell}
\usepackage{hyperref}
\usepackage{booktabs}
\usepackage{tabularx}
\usepackage{enumerate}
\usepackage{amssymb,mathtools}

\newtheorem{theorem}{Theorem}
\newtheorem{lemma}{Lemma}
\newtheorem{proposition}{Proposition}
\newtheorem{corollary}{Corollary}
\theoremstyle{definition}
\newtheorem{assumption}{Assumption}

\let\vec=\boldsymbol

\usepackage{tikz}
\usetikzlibrary{matrix,calc,positioning,fit,backgrounds,shadows.blur,arrows.meta}
\usepackage{lipsum}
\usetikzlibrary{matrix,calc,positioning,fit,backgrounds,shadows.blur,arrows.meta}
\usepackage{amsmath}
\usepackage[capitalize,noabbrev]{cleveref}

\definecolor{sensorcolor}{RGB}{65,105,225}
\definecolor{framecolor}{RGB}{220,20,60}
\definecolor{shapecolor}{RGB}{34,139,34}
\definecolor{navcolor}{RGB}{255,140,0}
\definecolor{band}{RGB}{246,246,246}
\definecolor{Crimson}{rgb}{0.86, 0.08, 0.24}

\tikzset{
  mod/.style   ={rectangle,rounded corners=2pt,draw,very thick,align=center,
                 minimum width=4.2cm,minimum height=2.0cm,fill=#1!14},
  nav/.style   ={rectangle,rounded corners=2pt,draw,very thick,align=center,
                 minimum width=14.0cm,minimum height=2.0cm,fill=#1!14},
  io/.style    ={circle,draw,thick,inner sep=0pt,minimum size=2.8mm,fill=white},
  qarr/.style  ={thick,dashed,-{Stealth[length=2.2mm]}},
  rarr/.style  ={thick,-{Stealth[length=2.2mm]}},
  chip/.style  ={rectangle,rounded corners=1pt,draw,thin,inner sep=1pt,font=\tiny\bfseries,fill=white},
  dsicon/.style={rectangle,draw,thin,fill=white,minimum width=4.5mm,minimum height=4.5mm},
  badge/.style ={circle,draw,thin,fill=white,minimum size=5mm,font=\tiny\bfseries},
  tiny/.style  ={font=\scriptsize},
  xtiny/.style ={font=\tiny}
}

\newcommand{\Ham}{\mathcal{H}}
\newcommand{\Pos}{q}
\newcommand{\Mom}{p}
\newcommand{\Mass}{\mathrm{M}}
\newcommand{\cvar}[1]{\mathrm{CVaR}_{#1}}

\title{Learning Material-Aware Hamiltonian Risk Fields for Safe Navigation}

\author{%
Aditya Sai Ellendula$^{1}$, Yi Wang$^{2}$, Chandrajit Bajaj$^{1,2}$ \\
$^{1}$ Department of Computer Science, University of Texas at Austin, Austin, TX 78712, USA \\
$^{2}$ Oden Institute, University of Texas at Austin, Austin, TX 78712, USA \\
\texttt{\{adityase,panzer.wy\}@utexas.edu, bajaj@cs.utexas.edu} \\
}

\begin{document}

\maketitle

\begin{abstract}
Risk-aware navigation should be selective: a policy should expose evasive degrees of freedom only when the local scene admits a lower-risk feasible maneuver, and suppress them when no safer alternative exists. We show that adding one context-energy term to a port-Hamiltonian navigation policy produces a learned force channel with exactly this falsifiable signature. When the local risk field contains a feasible lower-risk direction, the induced context force activates toward it; when the apparent escape is blocked or not yet available, a route-aware gate suppresses lateral force rather than hallucinating an unsafe maneuver. A CVaR tail-risk objective focuses gradient updates on rare but consequential risk transitions. We validate the selectivity signature across four settings. In the primary delayed-required-escape benchmark, route-aware CVaR reduces premature force activation from 0.950 to 0.180 versus DWA while raising success from 0.480 to 0.810 with zero replans. On real off-road terrain (RELLIS-3D), route-aware enrichment achieves correct activation rate 0.837 and false activation rate 0.114, compared to 0.378/0.752 for scalar risk gradients. On static semantic maps (DFC2018), enrichment reduces catastrophic failure from 0.60 to 0.10 and oscillation by 90.7\% while preserving path efficiency. In highway traffic, collisions drop from 100\% to 0\% when a lane escape is feasible; when no escape exists, the policy suppresses the lateral maneuver. The selectivity property follows from the gradient structure of the context energy rather than from training-time tuning.
\end{abstract}

\section{Introduction}
\label{sec:intro}

Risk-aware navigation has a selectivity gap. When a lower-risk feasible
maneuver exists, the robot should expose the evasive degree of freedom and
take it; when the apparent escape is blocked, corrupted, or not yet open, it
should suppress that same degree of freedom rather than hallucinate a detour.
Classical planners~\citep{hart1968formal,karaman2011sampling},
MPC/MPPI-style samplers~\citep{williams2017mppi}, reactive controllers, and
learned policies~\citep{schulman2017proximalpolicyoptimizationalgorithms,haarnoja2018softactorcriticoffpolicymaximum}
all address parts of this problem, but none make this
activation/suppression asymmetry an explicit structural property: risk costs
tend either to pull everywhere a gradient exists, causing false maneuvers, or
to be damped enough that the system misses real escape opportunities.

The issue is acute in off-road terrain and local traffic. A patch that looks
clear from geometry alone may be wet clay; a corridor that looks low-risk in
the semantic map may be physically blocked; a lane change may be useful behind
a slow leader but unsafe when the adjacent lane is boxed in. The decision
problem is therefore not just to reweight a scalar cost. The agent must decide
which force channels should be active under the current local context, and
which should remain latent.

Our key observation is that a single additive Hamiltonian energy term can make
this selectivity structural. Adding $\Ham_{\rm ctx}$ to a geometry-only
port-Hamiltonian policy simultaneously creates (i) a context force channel
in the momentum update, (ii) a parameter subspace and sensitivity stream for
learning when that channel matters, and (iii) a route by which CVaR gradients
from rare failures act directly on the closed-loop field.
Proposition~\ref{prop:enrichment} formalizes this coupling: enrichment is not merely a
new cost term, but a coordinated enlargement of the force field, learning
pathway, and update timescale.

The resulting policy remains a zero-replan local field. New sensory
information changes the cotangent update, not an external planner or critic.
A local affordance gate tests only the currently sensed patch: when a feasible
lower-risk primitive exists, the soft-risk force bends the field toward it;
when the candidate is blocked or insufficiently better, the soft channel is
suppressed and the rollout stays near the geometry-only policy.

\begin{figure*}[t]
    \centering
    \includegraphics[width=0.98\textwidth]{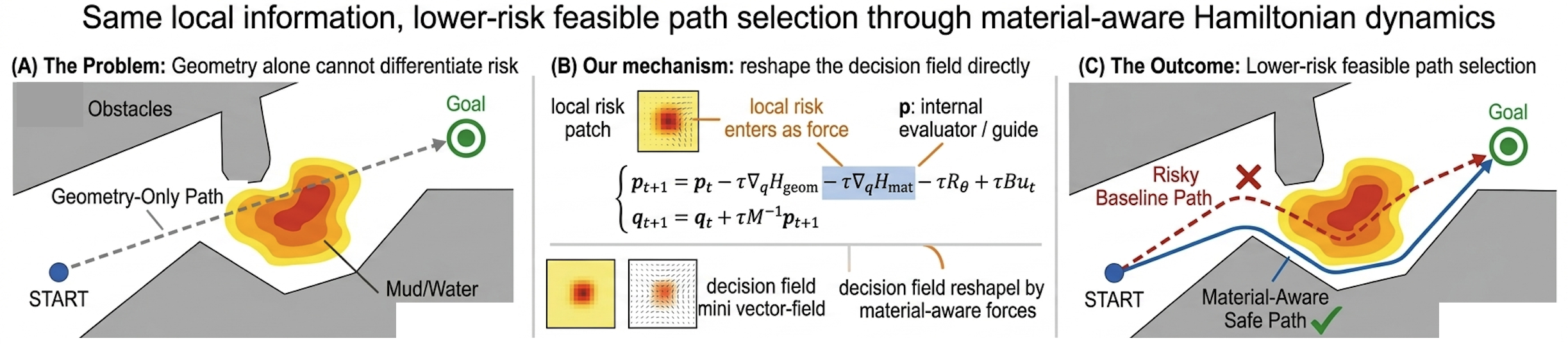}
    \caption{
    \textbf{Selective reshaping of the decision field.}
    \textbf{(A)} Geometrically feasible maneuvers can differ in material
    risk. \textbf{(B)} Adding
    \( -\tau \nabla_{\Pos} \Ham_{\mathrm{ctx}} \) to the cotangent update
    creates a context-force channel. \textbf{(C)} The channel bends toward
    a safer lane when one is feasible, but remains negligible when escape is
    boxed in. Sec.~\ref{sec:experiments} measures this activation/suppression
    signature directly.
    }
    \label{fig:teaser}
\end{figure*}

Prior Hamiltonian and port-Hamiltonian learners~\citep{greydanus2019hamiltonian,desai2021porthamiltonian,ellendula2025grlsnam}
provide structured vector fields, but their geometry-only policy is blind to
material/context risk. We extend that policy with soft-risk and hard-hazard
force channels and train with CVaR to emphasize rare consequential
transitions~\citep{chow2015risk}. The primary head-to-head benchmark is a
delayed local escape: the lateral maneuver is blocked early, necessary later,
and must be timed without global replanning.

Our contributions are:
\begin{itemize}
  \item \textbf{Selective risk-field enrichment.} One additive
        context-energy term creates a force channel that activates
        when the local scene admits a lower-risk feasible maneuver and
        suppresses when it does not. This turns risk adaptation into a
        context-conditioned change in the closed-loop field, rather
        than a uniform reweighting of a scalar cost.

  \item \textbf{A learnable energy enrichment principle.} We prove
        (Proposition~\ref{prop:enrichment}) that, within the class of
        port-Hamiltonian learners considered here, a new additive
        energy term induces a new force channel, parameter subspace,
        sensitivity stream, and update timescale. The context-enriched
        enlargement is structurally consistent rather than ad hoc.

  \item \textbf{Falsifiable evaluation across four settings.} We lead with a
        delayed-required-escape benchmark where the method reduces premature
        activation while increasing success with zero replans, then measure
        spatial selectivity (CAR/FAR/SR), temporal selectivity (false
        pre-activation, suppress rate, reaction delay), and empirical
        tail-risk reduction (violation CVaR) across RELLIS-3D, RELLIS-Dyn,
        DFC2018, and highway-env.
\end{itemize}

The central message of the paper is not that the learner becomes
risk-aware. It is that the learner becomes risk-aware
\emph{selectively}: reshaping the decision field exactly when the
geometry of the situation permits a safer alternative, exposing only the
useful degrees of freedom, and leaving the geometry-only behavior unchanged otherwise.

\section{Related Work}
\label{sec:related}

\paragraph{Structured dynamics and Hamiltonian learning.}
Hamiltonian and port-Hamiltonian learning provide the closest mechanistic
precedent for our view of policy as a structured phase-space flow.
Hamiltonian neural networks learn a scalar energy whose symplectic
gradient defines the vector field~\citep{greydanus2019hamiltonian};
SymODEN extends this idea to controlled systems~\citep{zhong2020symplectic};
Lagrangian networks and port-Hamiltonian neural networks add coordinate
or dissipation structure for physical system identification and
control~\citep{cranmer2020lagrangian,desai2021porthamiltonian}.
Classical port-Hamiltonian theory and energy-shaping control likewise
show how damping, interconnection, and external ports can be used to
shape closed-loop behavior~\citep{vanderschaft2014port}.
These works motivate our parameterization, but their usual target is
dynamics identification or stabilizing control from full-state data.
In contrast, we use the Hamiltonian structure as a navigation policy
class and add material-risk terms that act directly as learned force
channels under local sensing.

\paragraph{Safe and risk-sensitive reinforcement learning.}
Safe RL typically introduces risk as an objective- or constraint-level
quantity. CVaR optimization provides the variational foundation for
tail-risk objectives~\citep{rockafellar2000optimization}; CVaR MDPs and
risk-constrained actor-critic methods optimize tail or percentile risk
over trajectories~\citep{chow2015risk,chow2017risk}; and constrained
policy optimization enforces CMDP-style cost constraints during policy
search~\citep{achiam2017constrained}. Distributional RL instead models
the return distribution and can apply risk distortions to that
distribution~\citep{bellemare2017distributional,dabney2018implicit}.
Our objective uses the same tail-risk lineage, but the object being
updated is different: the CVaR signal trains coefficients of a structured
decision field, so risk changes the generated forces rather than only a
generic policy parameter vector, Lagrange multiplier, or return statistic.

\paragraph{Navigation, local planning, and safety filters.}
Classical navigation methods adapt paths or local actions through
search, sampling, or hand-designed potentials. A* and sampling-based
planners provide strong geometric references when map access is
available~\citep{hart1968formal,karaman2011sampling}; artificial
potential fields, the Dynamic Window Approach, MPPI, and CBF-QP filters
produce reactive local behavior under limited information
~\citep{khatib1986real,fox1997dynamic,williams2017mppi,ames2016control}.
These systems are important baselines because they also operate through
local costs or forces, but their risk terms are usually hand-designed,
externally scored, or enforced as filters. Our method keeps the local
closed-loop character while learning how semantic/material risk enters
the phase-space update.

\paragraph{Material-aware perception and traversability.}
Terrain and traversability methods often learn semantic classes,
freespace, or cost maps that are then consumed by a planner or
controller. Off-road datasets such as RUGD and RELLIS-3D support this
line of work~\citep{rugd2019,rellis2020}, and systems such as TerraPN
learn terrain costs for downstream local planning~\citep{sathyamoorthy2022terrapn}.
The DFC2018 Houston data provide a complementary remote-sensing setting
with rich land-cover cues~\citep{dfc2018}. Our DFC experiments use these
material labels to construct risk fields, but the contribution is not a
new semantic mapper: it is a mechanism for turning material risk into
internal energy and force terms that reshape the decision field.

\paragraph{Continual and multi-timescale adaptation.}
A final related thread studies adaptation over multiple timescales:
actor-critic methods separate critic and actor updates, continual
learning protects or reuses knowledge across tasks, and rapid motor
adaptation conditions policies on recent history
~\citep{kirkpatrick2017ewc,rusu2016progressive,kumar2021rma}.
Our three loops have a similar motivation--fast local correction,
episodic tail-risk updates, and slower curriculum advancement--but the
state being adapted is explicitly structured: local corrections and
episodic gradients modify Hamiltonian coefficients and material force
channels rather than an unconstrained black-box policy.

Prior work therefore covers each ingredient in isolation: structured
energy-based dynamics, risk-sensitive objectives, local planners and
safety filters, traversability maps, and multi-timescale adaptation. Our
contribution is their combination at the point where risk enters the
closed loop: a material-energy enrichment that reshapes the Hamiltonian
decision field itself.

\section{Method}
\label{sec:method}

We use two model names to avoid temporal-stage ambiguity. The
\emph{geometry-only policy} is the port-Hamiltonian navigator trained
without context risk;
the \emph{context-enriched field} is the same update after adding
one context-energy term and, when used, the local affordance gate.
Thus the enrichment is an architectural and training extension, not
the second temporal stage of an online rollout. The added term introduces
a learned force channel whose gradient reshapes the closed-loop decision
field only when the local scene admits a lower-risk feasible maneuver.  At each step the agent observes goal
displacement $\vec c_g-\vec c_t$, obstacle features within radius $\hat d$,
and a $P\times P$ local context patch $P_t$ containing a smoothed soft-risk
field $\tilde r(\cdot)\in[0,1]$ and a signed-distance field $\phi(\cdot)$ to
hard hazards.  These fields are bilinearly resampled at the current position
during integration.  All baselines and ablations use the same local-information
contract.

\subsection{Shared Hamiltonian update}
\label{sec:method:skeleton}

The geometry-only policy and the context-enriched field share the same phase-space update; the only dynamical
addition is the boldfaced context-force term in the momentum equation:
\begin{equation}
\renewcommand{\arraystretch}{1.5}
\begin{array}{ll}
\textbf{Geometry-only policy} & \textbf{Ctx-enriched} \\[2pt]
\Mom_{t+1} = \Mom_t
  - \tau\nabla_\Pos\Ham_{\rm geom}
  - \tau R_\theta(\Mom_t)
&
\Mom_{t+1} = \Mom_t
  - \tau\nabla_\Pos\Ham_{\rm geom}
  \;\boldsymbol{-\tau\nabla_\Pos\Ham_{\rm ctx}}
  - \tau R_\theta(\Mom_t) \\[4pt]
\multicolumn{2}{l}{\Pos_{t+1} = \Pos_t + \tau\Mass^{-1}\Mom_{t+1}
  \quad\text{(identical in both fields).}}
\end{array}
\label{eq:method:skeleton}
\end{equation}
New sensory information therefore enters through the cotangent variable
$\Mom$, not through a separate critic or global replanner.  The decision field
is the policy: changing the stored energy changes the vector field that
generates the next motion.

\subsection{Context energy and route-aware force channel}
\label{sec:method:energy}

The context term is additive, so the context-enriched field preserves the geometry-only policy
when its context coefficients vanish:
\begin{equation}
\Ham_\theta(\Pos,\Mom)
=
\underbrace{
  \tfrac{1}{2}\Mom^\top\Mass^{-1}\Mom
  + \beta D_{\mathcal{G}}(\Pos,\vec c_g)
  + \textstyle\sum_j\alpha_j b_{\rm IPC}(d_j)
}_{\Ham_{\rm kin}+\Ham_{\rm geom}\;\text{(geometry-only)}}
+
\underbrace{
  \lambda_s\tilde r(\vec c)
  + \lambda_h b_{\rm sp}(\phi(\vec c))
}_{\Ham_{\rm ctx}\;\text{(Ctx-enriched only)}}.
\label{eq:method:energy}
\end{equation}
Here $b_{\rm IPC}$~\citep{li2020ipc} is the incremental-potential contact
barrier and $b_{\rm sp}(\phi)=k^{-1}\log(1+e^{k(\hat d_\phi-\phi)})$ is a
softplus-relaxed inverse SDF.  Differentiating $\Ham_{\rm ctx}$ yields two
force channels:
\begin{equation}
F_{\rm ctx}(\vec c)
=
\underbrace{-\lambda_s\nabla\tilde r(\vec c)}_{F_{\rm soft}:\;\text{lateral risk deflection}}
+
\underbrace{-\lambda_h b_{\rm sp}'(\phi(\vec c))\nabla\phi(\vec c)}_{F_{\rm hard}:\;\text{hazard repulsion}},
\qquad
b_{\rm sp}'(\phi)=-\sigma(k(\hat d_\phi-\phi)).
\label{eq:method:force}
\end{equation}
$F_{\rm soft}$ expresses preferences among feasible maneuvers, while
$F_{\rm hard}$ is a differentiable penalty against boundary contact. It is not
a control-barrier certificate; throughout the paper, violation metrics should
be read as empirical tail-risk reduction under finite rollouts. Risk is
therefore not a post-hoc cost, but a force that reshapes the local closed-loop
dynamics.

\begin{figure}[t]
  \centering
  \includegraphics[width=0.95\linewidth]{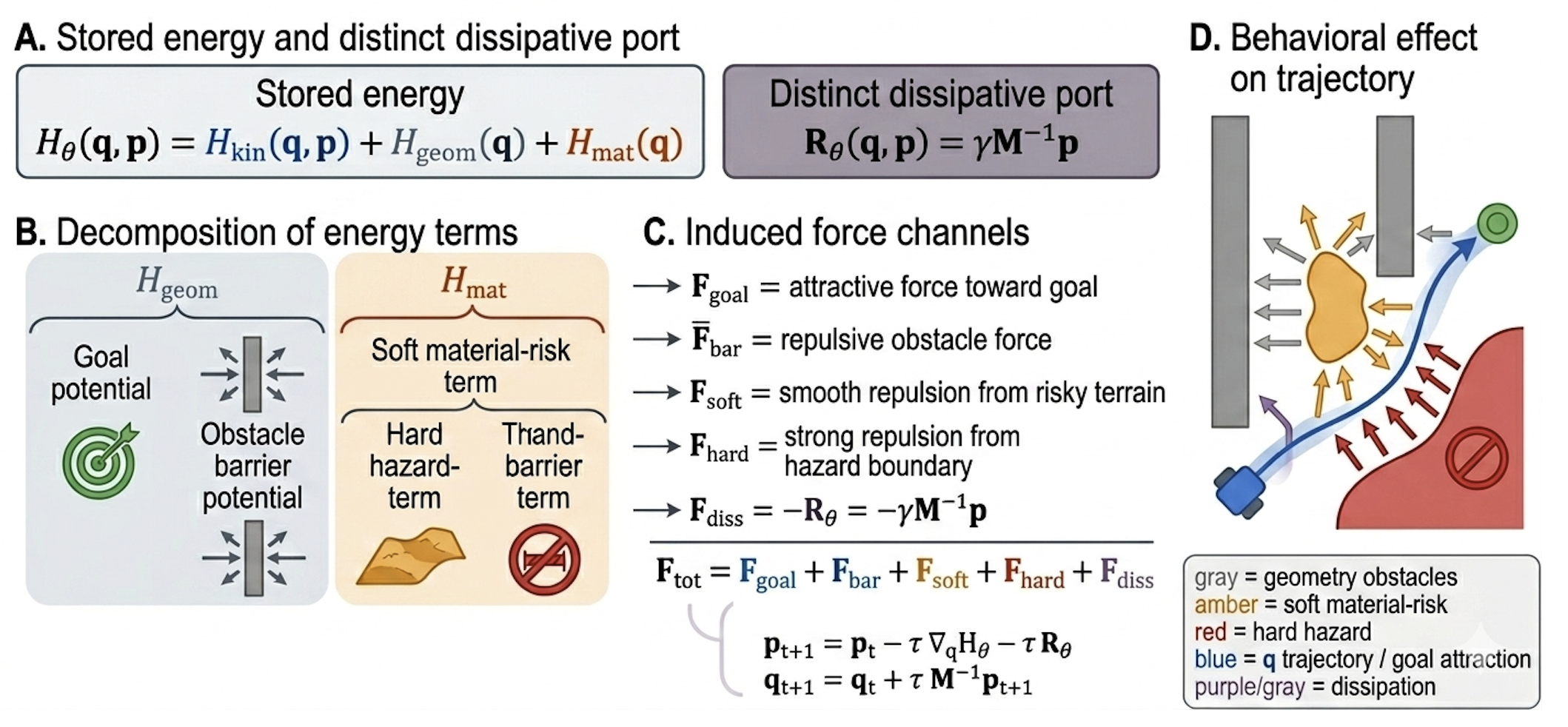}
  \caption{\textbf{Factored stored energy and induced force channels.}
  $\Ham_\theta$ separates kinetic, geometric, dissipative, and context
  terms. The context term creates a soft-risk deflection channel and a
  hard-hazard repulsion channel. The route-aware gate lets the soft channel
  shift the field only when a feasible lower-risk maneuver exists; otherwise
  the rollout stays near the geometry-only policy.}
  \label{fig:factored}
\end{figure}

The scalar soft-risk force can overreact when a low-risk-looking region is
blocked by a fence, berm, vehicle, or other non-traversable boundary.  We
therefore gate only the soft-risk channel using a local affordance test:
\begin{equation}
F_{\rm ctx}^{\rm ra}(\vec c;P_t)
=
m_{\rm feas}(\vec c,z_t)\bigl(-\lambda_s\nabla\tilde r(\vec c)\bigr)
-
\lambda_h b_{\rm sp}'(\phi(\vec c))\nabla\phi(\vec c),
\label{eq:method:route_aware_force}
\end{equation}
where $z_t$ denotes the finite set of $K$ short-horizon primitives sampled
inside the current BEV patch.  Let $k^\star$ be the feasible primitive with
lowest soft-risk exposure, and let $k_{\rm geo}$ be the primitive closest to
the geometry-only policy direction.  The gate is
\begin{equation}
m_{\rm feas}(\vec c,z_t)
=
\sigma\!\left(\kappa_R[(R_{k_{\rm geo}}-R_{k^\star})-\rho_R]\right)
\sigma\!\left(\kappa_C(C_{k^\star}-\delta_\phi)\right)
I_{k^\star}.
\label{eq:method:route_gate}
\end{equation}
Here $R_k$ is the primitive's soft-risk integral, $C_k$ its minimum SDF
clearance, and $I_k$ its local traversability indicator.  Thus
$F_{\rm soft}$ activates only when the current patch contains a feasible
candidate that improves risk by margin $\rho_R$; otherwise it is suppressed
without weakening $F_{\rm hard}$.  The gate is not a global replanner: it only
tests local affordance inside the same sensed patch available to all
local-sensing methods. Under semantic or risk-map corruption, the soft channel
therefore has a conservative failure mode: a noisy low-risk patch is not enough
to activate lateral deflection unless a short-horizon primitive also clears the
SDF and improves risk by margin $\rho_R$. This is not a perception guarantee,
but it tends to turn uncertain or inconsistent risk evidence into suppression
rather than an unconstrained evasive force.  Figure~\ref{fig:route_gate_main}
summarizes the gate used in all experiments.  Coefficients $(\lambda_s,\lambda_h)$ are predicted
by sigmoid-bounded CNN heads from $P_t$ fused with the goal token; geometric
heads $(\alpha,\beta,\gamma)$ are frozen from the geometry-only policy, so these coefficients
are the only route by which context alters the field
(Appendix~\ref{app:coeff_pred}).

\begin{figure}[t]
\centering
\fbox{\begin{minipage}{0.94\linewidth}
\small
\textbf{Route-aware soft-channel gate.}
Given the current BEV patch, sample short-horizon primitives $z_t$ and compute
each primitive's soft-risk integral $R_k$, minimum SDF clearance $C_k$, and
traversability flag $I_k$. Choose the lowest-risk feasible primitive $k^\star$
and compare it with the primitive $k_{\rm geo}$ closest to the geometry-only
policy direction. The soft channel opens only if
\[
R_{k_{\rm geo}}-R_{k^\star}>\rho_R,\qquad
C_{k^\star}>\delta_\phi,\qquad I_{k^\star}=1.
\]
If any test fails, $m_{\rm feas}\approx0$: $F_{\rm soft}$ is suppressed while
$F_{\rm hard}$ remains active. The test is local, differentiable through the
sigmoids in Eq.~\eqref{eq:method:route_gate}, and does not call a global planner.
\end{minipage}}
\caption{\textbf{Gate specification in the main method.}
The gate converts a risk-map cue into force activation only when the current
local patch contains a cleared, traversable primitive that improves soft risk
by margin $\rho_R$.}
\label{fig:route_gate_main}
\end{figure}

\subsection{Tail-risk objective}
\label{sec:method:cvar}

Each rollout accumulates
\begin{equation}
J(\theta) =
w_g\|\Pos_T-\Pos_g\|^2
+ w_\ell\textstyle\sum_t\|\Pos_{t+1}-\Pos_t\|
+ w_r\textstyle\sum_t\tilde r(\Pos_t)\|\Pos_{t+1}-\Pos_t\|
+ w_h\textstyle\sum_t\mathbf{1}[\phi(\Pos_t)<\epsilon].
\label{eq:method:J}
\end{equation}
Because the relevant failures are rare, expected-cost training can be dominated
by typical rollouts.  We therefore optimize the empirical
Rockafellar--Uryasev CVaR objective~\citep{rockafellar2000optimization}:
\begin{equation}
\widehat{\cvar{\alpha}}(\theta)
=
\hat\eta
+\frac{1}{(1-\alpha)B}
\sum_{i=1}^{B}\bigl(J^{(i)}(\theta)-\hat\eta\bigr)_+,
\qquad
\hat\eta=\widehat Q_\alpha\!\bigl(\{J^{(i)}\}_{i=1}^B\bigr)\;\text{detached}.
\label{eq:method:cvar_emp}
\end{equation}
With $\alpha=0.95$ and $B=64$, gradient flows through the worst rollouts,
giving a consistent estimator~\citep{hong2014monte}.  The resulting gradient
acts through the sensitivity
$\partial \Pos_{t+1}/\partial(\lambda_s,\lambda_h)$ exactly on rollouts where
the context force can change the outcome.

\subsection{Training hierarchy}
\label{sec:method:loops}

The context-enriched field uses three computational roles.  First, the geometry-only checkpoint is
loaded with $\lambda_s,\lambda_h\equiv0$, so Eq.~\eqref{eq:method:skeleton}
initially reproduces the geometry-only policy.  Second, the episode loop
backpropagates the CVaR loss in Eq.~\eqref{eq:method:cvar_emp} through the
differentiable rollout to update the context encoder and coefficient heads.
Third, a lightweight segment-correction loop updates active geometry/context
coefficients from realized local residuals, while a curriculum loop advances
to harder patches only after validation violation rates fall below a threshold.
The full loop diagram, initialization protocol, and pseudocode are given in
Appendix~\ref{app:ctx_loops_hierarchy}; the main text keeps the algorithmic
contract to the force law, loss, and update hierarchy above.
For the logged $500$-epoch highway runs, the geometry-only policy trained for
$1.49$ hours and the context-enriched policy for $3.27$ hours on one NVIDIA
A100 GPU, for $4.8$ total GPU-hours.

\subsection{Selectivity consequences of context-energy enrichment}
\label{sec:method:enrichment}

The additions above are coupled: adding $\Ham_{\rm ctx}$ creates the force
channel $F_{\rm ctx}=-\nabla_\Pos\Ham_{\rm ctx}$, the coefficient subspace
$(\lambda_s,\lambda_h)$, and the rollout-sensitivity pathway through which
CVaR training acts.  The same structure yields the selectivity consequences
measured in Sec.~\ref{sec:experiments}.

\begin{proposition}[Selectivity consequences of context-energy enrichment]
\label{prop:enrichment}
Let the geometry-only policy and the context-enriched field share the same geometric Hamiltonian, dissipative
port $R_\theta$, mass matrix $\Mass$, and semi-implicit integration rule.  Let
$\Ham_{\rm enr}=\Ham_{\rm geo}+\Ham_{\rm ctx}$ with
$F_{\rm ctx}=-\nabla_\Pos\Ham_{\rm ctx}$.  Under the boundedness and local
Lipschitz assumptions in Appendix~\ref{app:theory}, the following hold over
any finite horizon:
\begin{enumerate}
  \item[\textbf{C1.}] \emph{Geometry-only preservation.} If
  $\sup_t\|F_{\rm ctx}(q_t)\|\le\varepsilon$, then the context-enriched rollout remains
  within $O(\varepsilon)$ of the geometry-only rollout.
  \item[\textbf{C2.}] \emph{No hallucinated escape.} If
  $\sup_t\|P_\perp F_{\rm ctx}(q_t)\|\le\varepsilon_\perp$, then lateral
  deviation from the geometry-only rollout is bounded by $O(\varepsilon_\perp)$.
  \item[\textbf{C3.}] \emph{Selective risk deflection.} If a feasible lateral
  direction has projected risk-gradient margin $\Delta$ and projected
  hard-barrier forces cancel the soft-risk force by at most $\chi$, then
  $\|P_\perp F_{\rm ctx}\|\ge\lambda_s\Delta-\chi$ and one semi-implicit step
  decreases local soft risk relative to the geometry-only policy by
  $O(\tau^2\lambda_s\Delta^2)$ whenever $\lambda_s\Delta>\chi$.
\end{enumerate}
\end{proposition}

Proofs: C1 $\Rightarrow$ Theorem~\ref{thm:scaffold_preservation};
C2 $\Rightarrow$ Corollary~\ref{cor:no_hallucination};
C3 $\Rightarrow$ Theorem~\ref{thm:risk_deflection}.  Selectivity is not
trained as a separate label: C1--C3 follow from the gradient structure of
$\Ham_{\rm ctx}$ and the affordance gate, not from post-hoc tuning. Appendix
Table~\ref{tab:app_theory_checks} makes C1--C3 checkable by mapping each
theoretical consequence to measured rollout statistics.

\begin{table}[t]
\centering
\caption{\textbf{Main-text check of the theory predictions.}
We test C1--C3 as finite-rollout predictions, not formal safety certificates.
``Below-$\delta$'' means lateral deviation stays below the reaction threshold
used for false pre-activation.}
\label{tab:main_theory_checks}
\small
\setlength{\tabcolsep}{4pt}
\renewcommand{\arraystretch}{1.06}
\begin{tabular}{lccc}
\toprule
Prediction & Rollout condition & Below-$\delta \uparrow$ & Failure metric$\downarrow$ \\
\midrule
C1 & geometry-only delayed escape
& $1.000$ & false pre-act $0.000$ \\
C2 & route-aware context, static blocked R2
& $0.886$ & FAR $0.114$ \\
C3 & route-aware CVaR, pre-escape delayed
& $0.820$ & false pre-act $0.180$ \\
\bottomrule
\end{tabular}
\end{table}

\section{Experiments: Selective Hamiltonian Risk Adaptation}
\label{sec:experiments}

The experiments ask one question across four settings: does Hamiltonian
context enrichment activate risk forces when a safer feasible alternative
exists and suppress them when it does not? We lead with the head-to-head
domain where this mechanism is decisive: delayed required escape, where
the policy must suppress a tempting lateral maneuver before it is feasible
and activate it after the escape opens, all with zero replans.

\subsection{Setup, Metrics, and Domains}
\label{sec:exp_setup}

\paragraph{Research questions.}
RQ1--RQ2 (spatial): does the context-enriched field activate toward a feasible lower-risk
direction and suppress when blocked?
RQ3--RQ5 (temporal): does the context-enriched field activate faster than replanning
baselines, avoid pursuing closed or unavailable escapes, and
suppress-then-activate within one episode?
RQ6: must risk enter the force channel, or is loss reweighting sufficient?
RQ7: does CVaR training reduce tail violations, not only mean risk?

\paragraph{Regimes.}
All domains share \textbf{R1} (feasible lower-risk alternative exists),
\textbf{R2} (lower-risk-looking direction blocked or dynamically unsafe),
and \textbf{R3} (risk-neutral; the context-enriched field should preserve geometry-only behavior).
RELLIS-Dyn extends R1/R2 from space to time via eight event types spanning
soft-risk, hard-boundary, dynamic-obstacle, and compound/delayed scenarios;
full definitions are in Appendix~\ref{app:rellis_dyn_details}. The primary
head-to-head benchmark is \emph{delayed required escape}: blocked before
$t_{\rm escape}$, necessary after.

\paragraph{Metrics.}
Spatial selectivity:
\begin{equation}
\mathrm{CAR}=\Pr[\langle F_{\rm ctx},d_{\rm safe}\rangle{>}\epsilon\mid R1],\;\;
\mathrm{FAR}=\Pr[|P_\perp F_{\rm ctx}|{>}\epsilon\mid R2],\;\;
\mathrm{SR}=\frac{\mathbb{E}[|P_\perp F_{\rm ctx}|\mid R1]}
                 {\mathbb{E}[|P_\perp F_{\rm ctx}|\mid R2]},
\end{equation}
and AUPRC (threshold-independent). Temporal: \emph{reaction delay}
($\delta$-displacement from the pre-event geometry-only trajectory; sensitivity in
Appendix Table~\ref{tab:app_rellis_dyn_delay_sensitivity}),
\emph{stale exposure} (violation before reaction),
\emph{false pre-activation} (lateral force before $t_{\rm escape}$),
\emph{suppress rate} ($1-$false pre-activation).
Post-event \emph{violation CVaR} (hard contact weighted separately) is
the primary tail metric.

\subsection{Primary Head-to-Head: Delayed Required Escape}
\label{sec:exp_primary_delayed}

All methods receive the same updated BEV patch at every step; only the
decision mechanism differs. Table~\ref{tab:delayed_escape} is the main
head-to-head result: the route-aware CVaR field reduces DWA and semantic
MPPI false pre-activation from $0.950$ to $0.180$ while raising success
from $0.480$/$0.240$ to $0.810$, and does so without replans. The table also isolates the
three ingredients: Hamiltonian force structure, CVaR timing, and
context-conditioned coefficients.

\begin{table}[t]
\centering
\caption{\textbf{Delayed-required-escape} benchmark.
Escape is blocked before $t_{\rm escape}$ and necessary afterward. The full
model best suppresses early false activation while still succeeding after
the escape opens. Main rows use $100$ paired episodes. $^\dagger$Low CVaR
with success ${<}0.21$ indicates low exposure from getting stuck; CIs are
in Appendix Table~\ref{tab:app_delayed_escape_ci}.}
\label{tab:delayed_escape}
\small
\setlength{\tabcolsep}{4pt}
\renewcommand{\arraystretch}{1.1}
\begin{tabular}{lcccc}
\toprule
Method & False pre-act$\downarrow$ & Suppress$\uparrow$
       & Success$\uparrow$ & Viol.\ CVaR$\downarrow$ \\
\midrule
Geometry-only policy  & $0.000$ & $1.000$ & $0.030$ & $1.894$ \\
Risk-loss-only          & $0.420$ & $0.580$ & $0.040$ & $1.793$ \\
Fixed-coeff context field     & $0.990$ & $0.010$ & $0.030$ & $0.463^\dagger$ \\
Black-box CVaR policy   & $0.920$ & $0.080$ & $0.200$ & $0.503^\dagger$ \\
DWA semantic            & $0.950$ & $0.050$ & $0.480$ & $0.695$ \\
MPPI semantic           & $0.950$ & $0.050$ & $0.240$ & $0.831$ \\
Ctx-enriched, expected cost  & $0.370$ & $0.630$ & $0.620$ & $0.855$ \\
\textbf{Route-aware Ctx CVaR}
  & $\mathbf{0.180}$ & $\mathbf{0.820}$
  & $\mathbf{0.810}$ & $\mathbf{0.740}$ \\
\bottomrule
\end{tabular}
\end{table}

\begin{figure*}[t]
\centering
\includegraphics[width=\textwidth]{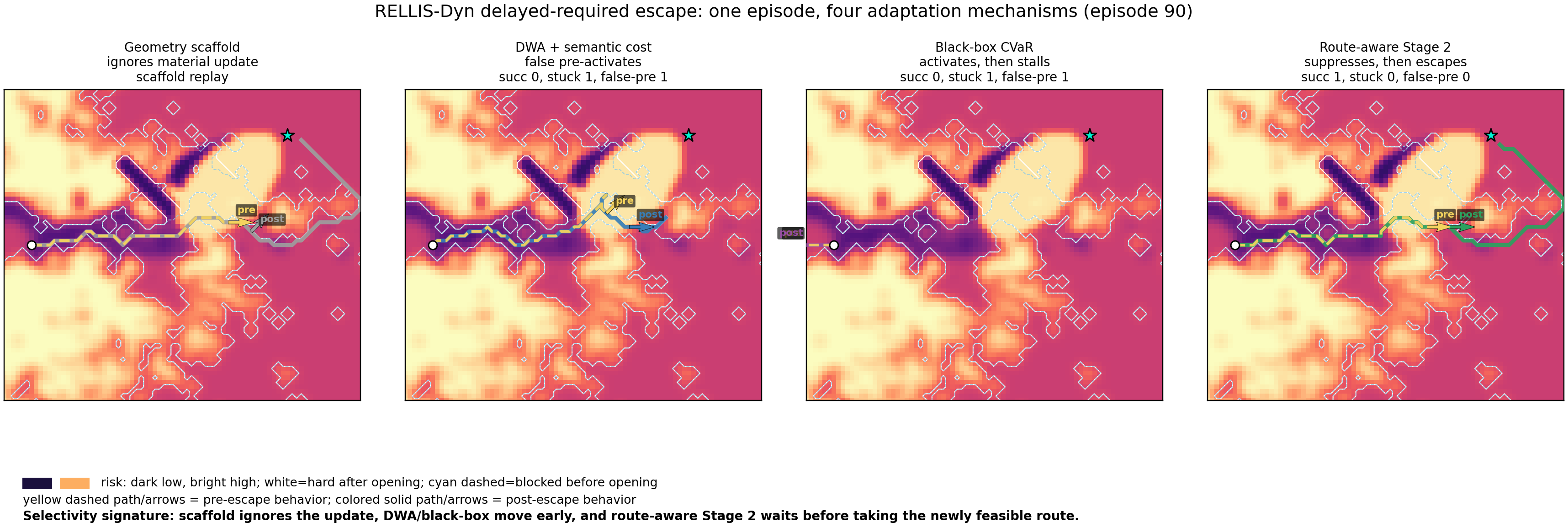}
\caption{\textbf{Qualitative temporal selectivity in one delayed-required
escape episode.} Yellow dashed trajectories/arrows show behavior before the
escape is available; solid colored trajectories/arrows show behavior after it
opens. The geometry-only policy ignores the material update, DWA and black-box CVaR move
before the escape is feasible and then stall, while route-aware context
enrichment suppresses before $t_{\rm escape}$ and takes the newly feasible
route after it opens.}
\label{fig:rellis_dyn_spotlight}
\end{figure*}

\emph{Hamiltonian structure}: black-box CVaR uses the same CVaR objective
and BEV patch but no Hamiltonian force channel; false pre-activation
$0.920$, success $0.200$ vs.\ $0.810$. Static CAR $0.884$ does not
transfer to temporal suppression.
\emph{CVaR training}: same architecture, different objective; CVaR halves
false pre-activation ($0.370{\to}0.180$), raises suppression
($0.630{\to}0.820$), improves success by $30\%$, reduces violation CVaR
($0.855{\to}0.740$).
\emph{Context-conditioning}: fixed vs.\ CNN-predicted $(\lambda_s,\lambda_h)$;
success $0.030$ vs.\ $0.810$. Without adaptive coefficients the method
suppresses nothing.

\subsection{RELLIS-3D: Spatial Selectivity}
\label{sec:exp_rellis_static}

Table~\ref{tab:rellis_selectivity} reports six variants on 2{,}250 BEV
episodes (leave-one-sequence-out; CIs in Appendix
Table~\ref{tab:app_rellis_full}). The SR gap between black-box CVaR ($1.292$) and route-aware context-enriched field
($2.358$) is the spatial analog of the temporal gap in
Table~\ref{tab:delayed_escape}: the black-box learns \emph{when} risk
gradients exist but not \emph{whether} following them is safe.
The route-aware gate produces the selectivity; Section~\ref{sec:exp_primary_delayed}
shows the same architecture without the gate (black-box) collapses to
$0.200$ navigation success on the strictest temporal test.

\begin{table}[t]
\centering
\caption{\textbf{RELLIS-3D spatial selectivity.}
Route-aware context enrichment achieves the best SR and AUPRC, indicating
selective activation rather than uniform risk-gradient following. Fixed
coefficients fail to suppress in blocked scenes (FAR~$0.888$).}
\label{tab:rellis_selectivity}
\small
\setlength{\tabcolsep}{5pt}
\renewcommand{\arraystretch}{1.1}
\begin{tabular}{lcccc}
\toprule
Method & CAR$\uparrow$ & FAR$\downarrow$ & SR$\uparrow$ & AUPRC$\uparrow$ \\
\midrule
Geometry-only policy    & $0.000$ & $0.000$ & $0.000$ & $0.000$ \\
Scalar context field            & $0.378$ & $0.752$ & $1.031$ & $0.008$ \\
Fixed-coeff context field    & $0.566$ & $0.888$ & $1.123$ & $0.008$ \\
Non-route directional Ctx  & $0.135$ & $0.178$ & $1.165$ & $0.206$ \\
Black-box CVaR policy     & $0.884$ & $0.129$ & $1.292$ & $0.206$ \\
\textbf{Route-aware Ctx} & $\mathbf{0.837}$ & $\mathbf{0.114}$
                           & $\mathbf{2.358}$ & $\mathbf{0.289}$ \\
\bottomrule
\end{tabular}
\end{table} 

\subsection{RELLIS-Dyn: Temporal Selectivity}
\label{sec:exp_rellis_dyn}

\paragraph{Material-event aggregate.}
Table~\ref{tab:rellis_dyn_agg} reports the three-event material subset
(mud onset, corridor closes, corridor opens) at 100 episodes per event.
Black-box CVaR success $0.273$ on the same task where it achieves
CAR~$0.884$ confirms that static selectivity metrics do not predict
navigation success.
Risk-loss-only hard exposure $0.694$ vs.\ Route-aware Ctx $0.027$
supports RQ6: the risk signal must enter the force channel.
The route-aware context-enriched field is the strongest zero-replan learned-field method,
matching planning baselines that use $9$--$68$ explicit replans.
Corridor-opens trajectory detail is in Appendix
Figure~\ref{fig:app_rellis_dyn_corridor}; the eight-event breakdown
(including moving-obstacle events where DWA is architecturally
stronger) is in Appendix Table~\ref{tab:app_rellis_dyn_8event_groups}.

\begin{table}[t]
\centering
\caption{\textbf{RELLIS-Dyn 3-event material aggregate.}
All methods receive the same updated BEV patch. The gate uses only
$K{=}8$ local primitives, not graph search or trajectory optimization.
$^\dagger$Oracle has global map access. $^{\ddagger}$Delay includes
gate-suppressed R2 episodes; see Table~\ref{tab:delayed_escape}.}
\label{tab:rellis_dyn_agg}
\small
\setlength{\tabcolsep}{3pt}
\renewcommand{\arraystretch}{1.1}
\resizebox{0.9\linewidth}{!}{%
\begin{tabular}{llccccccc}
\toprule
Method & Family
  & Succ.$\uparrow$ & Soft CVaR$\downarrow$
  & Delay$\downarrow$ & Stuck$\downarrow$
  & Replans$\downarrow$ & ms/step$\downarrow$ \\
\midrule
Geometry-only policy      & geom.-only     & $1.000$ & $0.712$ & $16.2$ & $0.04$ & $0$    & $0.8$  \\
Risk-loss-only     & learned loss   & $1.000$ & $0.673$ & $13.4$ & $0.03$ & $0$    & $0.9$  \\
Fixed-coeff Ctx     & learned field  & $0.120$ & $0.459$ & $1.47$ & $0.74$ & $0$    & $1.4$  \\
Black-box CVaR     & learned field  & $0.273$ & $0.482$ & $3.61$ & $0.61$ & $0$    & $4.2$  \\
Ctx-enriched, exp.cost  & learned field  & $0.983$ & $0.465$ & $1.30$ & $0.02$ & $0$    & $2.1$  \\
DWA semantic       & reactive       & $0.567$ & $0.621$ & $2.8$  & $0.43$ & $0$    & $0.6$  \\
CBF-QP             & safety filter  & $0.890$ & $0.651$ & $7.4$  & $0.09$ & $0$    & $0.9$  \\
Local A* (budget)  & planner        & $0.983$ & $0.629$ & $4.6$  & $0.01$ & $9.8$  & $8.4$  \\
MPC (budget)       & planner        & $0.967$ & $0.624$ & $3.8$  & $0.03$ & $11.1$ & $42.2$ \\
\textbf{Route-aware Ctx}& learned field & $\mathbf{0.980}$ & $\mathbf{0.638}$
  & $\mathbf{1.8}^{\ddagger}$ & $\mathbf{0.06}$ & $\mathbf{0}$ & $\mathbf{2.3}$ \\
Oracle replanner   & planner$^\dagger$  & $1.000$ & $0.601$ & $1.9$  & $0.00$ & $68.3$ & $97.5$ \\
\bottomrule
\end{tabular}}
\end{table}

\paragraph{Within-episode temporal selectivity.}
Figure~\ref{fig:rellis_dyn_spotlight} shows the delayed-required-escape
mechanism qualitatively.
Route-aware CVaR suppresses lateral activation throughout the blocked
window and activates within two steps of $t_{\rm escape}$
(false pre-activation $0.180$).
The expected-cost variant activates earlier ($0.370$), directly
demonstrating CVaR's tail-gradient contribution.
DWA false pre-activates on $95\%$ of episodes, typically $3$--$5$ steps
before $t_{\rm escape}$, because partial boundary changes trigger
its reactive response before the escape is fully open.

\subsection{DFC2018: Geometry-Only Repair}
\label{sec:exp_dfc}

Table~\ref{tab:dfc_geom_ctx} shows enrichment repairs a geometry-only
policy on a static semantic material map: success $0.867{\to}1.000$,
catastrophic failure $0.600{\to}0.100$, hard-hazard traversal $-96.9\%$,
cumulative risk $-59.2\%$, oscillation $-90.7\%$, path-length ratio
unchanged.
Full planner comparison in Appendix~\ref{app:dfc_full}; full-map
planners with global access obtain lower raw risk but the context-enriched field produces
smoother trajectories with zero oscillation.

\begin{table}[t]
\centering
\caption{\textbf{DFC2018 outcome shift} (300 paired episodes).}
\label{tab:dfc_geom_ctx}
\scriptsize\setlength{\tabcolsep}{3pt}\renewcommand{\arraystretch}{1.08}
\resizebox{0.7\linewidth}{!}{%
\begin{tabular}{lrrrrrr}
\toprule
Method & Succ.$\uparrow$ & Cat.fail$\downarrow$ & Hard len.$\downarrow$
       & Risk$\downarrow$ & Len.ratio & Osc.$\downarrow$ \\
\midrule
Geometry-only policy & $0.867$ & $0.600$ & $21.106$ & $23.267$ & $1.044$ & $82.212$ \\
\textbf{Ctx-enriched} & $\mathbf{1.000}$ & $\mathbf{0.100}$ & $\mathbf{0.658}$
                 & $\mathbf{9.493}$ & $1.045$ & $\mathbf{7.629}$ \\
\bottomrule
\end{tabular}}
\end{table}

\subsection{Highway-env: Interaction Selectivity}
\label{sec:exp_highway}

The open slow-leader scenario is the interaction analog of corridor
opens (lateral channel should activate); the boxed scenario is the
analog of moving-obstacle-blocks-detour (channel should suppress).
Table~\ref{tab:highway_mechanism} shows the mechanism ablation: removing
the lateral channel causes collision with the slow leader on every
episode; adding it without the TTC channel causes off-road failure in
the boxed scenario; only the full system passes and suppresses correctly.
MOBIL-IDM and risk-aware MPC outperform the context-enriched field overall (Appendix
Table~\ref{tab:app_highway_baselines}); the context-enriched field is the strongest
learned method and the only one exhibiting the intended
activation/suppression split.

\begin{table}[t]
\centering
\caption{\textbf{Highway mechanism ablation} (200 paired seeds).
The lateral channel enables passing; the TTC channel prevents using it
when the adjacent maneuver is unsafe.}
\label{tab:highway_mechanism}
\scriptsize\setlength{\tabcolsep}{3pt}\renewcommand{\arraystretch}{1.08}
\resizebox{\linewidth}{!}{%
\begin{tabular}{lrrrrrr}
\toprule
Variant & Def.off-rd$\downarrow$ & Slow coll$\downarrow$
        & Slow succ$\uparrow$ & Boxed coll$\downarrow$
        & Boxed off-rd$\downarrow$ & Boxed spd$\uparrow$ \\
\midrule
No lateral context       & $0.00$ & $1.00$ & $0.00$ & $1.00$ & $0.00$ & $22.33$ \\
Lateral only             & $0.20$ & $0.00$ & $1.00$ & $1.00$ & $0.00$ & $22.33$ \\
\textbf{Lateral + TTC}   & $\mathbf{0.00}$ & $\mathbf{0.00}$ & $\mathbf{1.00}$
                         & $\mathbf{0.00}$ & $\mathbf{0.00}$ & $\mathbf{7.57}$ \\
\bottomrule
\end{tabular}}
\end{table}

\subsection{Interpretation and Limitations}
\label{sec:exp_interpretation}

\paragraph{Tradeoff positions.}
Reactive baselines (DWA, CBF-QP) handle hard-boundary events well but
false pre-activate on discovery events (DWA: $0.950$ on delayed escape).
Planning baselines reduce violation under replan budget but pay
$9$--$68\times$ more replans.
The route-aware context-enriched field occupies the zero-replan gradient-following frontier:
strongest on soft-risk and escape-discovery events, weaker on
moving-obstacle events where reactive boundary following suffices.
The eight-event Pareto (Appendix Figure~\ref{fig:app_rellis_dyn_pareto})
shows the full tradeoff surface.

\paragraph{Limitations.}
(i) $b_{\rm sp}(\phi)$ is a differentiable barrier, not a CBF
certificate; our claims are empirical tail-risk reductions, not
forward-invariance guarantees
(Appendix~\ref{app:theory}, \S A.9).
(ii) Performance depends on risk-map quality; at $20$--$30\%$ label
corruption, CAR/FAR degrade despite the gate's tendency to suppress
inconsistent soft-risk evidence (Appendix
Table~\ref{tab:app_rellis_perception}).
(iii) Gate overhead must be profiled on latency-constrained hardware
(ms/step in Table~\ref{tab:rellis_dyn_agg}).


\section{Conclusion}
\label{sec:conclusion}

One additive context-energy term gives a port-Hamiltonian policy a falsifiable selectivity signature:
activate when a safer feasible alternative exists and suppress otherwise. Ablations isolate Hamiltonian
structure, CVaR tail timing, and context-conditioned coefficients; the pattern holds across off-road,
semantic, dynamic, and highway settings without global replanning.

\paragraph{Social Impact Statement.}
This work uses public or simulated navigation data and does not rely on private user information. Its intended benefit is more reliable autonomous navigation under terrain and traffic uncertainty. The method is not a deployment safety certificate: real systems would require independent validation, fallback control, and monitoring.

\paragraph{Acknowledgement.} This research was supported in part by a grant from the Peter O’Donnell Foundation, the Michael J Fox Foundation, Jim Holland- Backcountry Foundation and in part from a grant from the Army Research Office accomplished under Cooperative Agreement Number W911NF-19-2-0333

\bibliographystyle{plainnat}
\bibliography{references}

\newpage
\appendix

%
%

\section{Theoretical Guarantees for Contextual Hamiltonian Enrichment}
\label{app:theory}

This appendix collects the formal properties used in Section~3.6.
The goal is not to claim unconditional safety or global optimality. Instead, we prove
that contextual Hamiltonian enrichment has five conditional properties aligned with
the central claims of the paper: (i) an added energy term induces a new force and
sensitivity channel; (ii) the enriched system preserves the port-Hamiltonian
dissipativity structure; (iii) the discrete rollout is practically dissipative for small
enough step size; (iv) when context gradients vanish, the context-enriched field remains close to the
geometry-only policy and cannot generate a spurious lateral escape; and (v) when a
projected risk-gradient margin exists, the contextual force activates and locally
reduces risk. We also justify the detached-quantile CVaR gradient used in the
episode loop and state the constrained risk-averse OCP interpretation used by the
upper-confidence update. The passivity arguments follow the standard port-Hamiltonian energy
balance~\citep{vanderschaft2014port,desai2021porthamiltonian}; the CVaR argument
uses the Rockafellar--Uryasev variational representation~\citep{rockafellar2000optimization}
and standard empirical quantile sensitivity arguments~\citep{hong2014monte}.

\subsection{Setup and notation}
\label{app:setup}

Let \(q \in \mathbb R^d\) denote the configuration coordinate and \(p\in\mathbb R^d\)
its cotangent variable. Let \(x=(q,p)\), let \(M\succ 0\) be the mass matrix, and
write
\[
    v = M^{-1}p .
\]
The geometry-only Hamiltonian is
\[
    H_{\rm scaf}(q,p)
    =
    \frac12 p^\top M^{-1}p + H_{\mathrm{geom}}(q),
\]
where \(H_{\mathrm{geom}}\) contains the goal and geometric barrier terms. The context-enriched field
adds the context Hamiltonian
\[
    H_{\mathrm{ctx}}(q;z)
    =
    \lambda_s \widetilde r(\mathbf c;z)
    +
    \lambda_h b_{\mathrm{sp}}(\phi(\mathbf c;z)),
\]
where \(\mathbf c\) is the workspace position component of \(q\), \(z\) denotes the
locally sensed context patch, \(\widetilde r\) is the smoothed soft-risk field, and
\(\phi\) is the signed-distance field for hard hazards. Thus
\[
    H_{\rm enr}(q,p;z)
    =
    H_{\rm scaf}(q,p)+H_{\mathrm{ctx}}(q;z).
\]
The corresponding context force is
\[
    F_{\mathrm{ctx}}(q;z)
    =
    -\nabla_q H_{\mathrm{ctx}}(q;z)
    =
    -\lambda_s \nabla_q \widetilde r(\mathbf c;z)
    -
    \lambda_h b_{\mathrm{sp}}'(\phi(\mathbf c;z))\nabla_q \phi(\mathbf c;z).
\]
This is exactly the soft-risk plus hard-hazard force decomposition used in the main
paper.

The continuous-time enriched port-Hamiltonian dynamics are
\begin{align}
    \dot q &= \nabla_p H_{\rm enr}(q,p;z)=M^{-1}p, \label{eq:ct_q}\\
    \dot p &=
    -\nabla_q H_{\rm enr}(q,p;z)
    -
    D(q,p;z)M^{-1}p
    +
    B(q,p;z)u , \label{eq:ct_p}
\end{align}
where \(D(q,p;z)\succeq 0\) is the dissipative port. The rollout used in the paper is
the semi-implicit update
\begin{align}
    p_{t+1}
    &=
    p_t
    -\tau\nabla_q H_{\mathrm{geom}}(q_t)
    -\tau\nabla_q H_{\mathrm{ctx}}(q_t;z_t)
    -\tau D_tM^{-1}p_t
    +\tau B_tu_t, \label{eq:disc_p}\\
    q_{t+1}
    &=
    q_t+\tau M^{-1}p_{t+1}. \label{eq:disc_q}
\end{align}
When \(D_tM^{-1}p_t\) is implemented as the scalar dissipative port
\(R_\theta(q_t,p_t)=\gamma M^{-1}p_t\), this reduces to the update in the main text.

\begin{assumption}[Rollout regularity]
\label{ass:regularity}
There exists a compact rollout tube \(\mathcal K\subset \mathbb R^{2d}\) such that
the geometry-only and context-enriched trajectories remain in \(\mathcal K\) for the horizon under
analysis. On \(\mathcal K\), \(H_{\mathrm{geom}}\) and \(H_{\mathrm{ctx}}\) are \(C^2\),
and their gradients are Lipschitz.
\end{assumption}

\begin{assumption}[Bounded context force]
\label{ass:bounded_ctx}
There exists \(G_{\mathrm{ctx}}<\infty\) such that
\[
    \|\nabla_q H_{\mathrm{ctx}}(q;z)\|\le G_{\mathrm{ctx}}
\]
for all \((q,p)\in\mathcal K\) and all context patches considered.
\end{assumption}

\begin{assumption}[Dissipative port]
\label{ass:diss}
There exists \(\rho\ge 0\) such that
\[
    D(q,p;z)\succeq \rho I
\]
on \(\mathcal K\). When \(\rho>0\), the system is strictly dissipative in velocity.
\end{assumption}

\begin{assumption}[Local Lipschitz rollout map]
\label{ass:lipschitz_rollout}
The shared geometry-only update map is locally Lipschitz on \(\mathcal K\): there exists
\(L_\Phi\ge 0\) such that one step of the shared dynamics amplifies state errors by at
most \(1+\tau L_\Phi\).
\end{assumption}


\subsection{Full local-risk interface and problem setup}
\label{app:setup_full}

At each time step $t$ the agent observes a goal displacement
$\vec c_g-\vec c_t$, locally-sensed obstacle or neighbor features
$\{(C_j,R_j,W_j)\}_{j=1}^{N(t)}$ within sensing radius $\hat d$, and a
$P\times P$ context patch $P_t$ centered at the current position.
The patch contains a smoothed risk field $\tilde r(\cdot)\in[0,1]$ and a
signed-distance field $\phi(\cdot)$ to hard hazards or infeasible regions.
Depending on the domain, $\tilde r$ may encode material risk, traffic
interaction risk, or a procedurally constructed local interaction field;
the method only requires that the field be locally queryable and
differentiable after smoothing.
The patch is bilinearly resampled at the agent's current position at
every integration step, so forces and costs reflect the actual path
traversed.
All comparisons use the same local-information regime; differences come
from the decision mechanism, not from privileged sensing.

Each observation defines a local optimization instance.
The goal term, clearance barriers, context-risk preference, hazard
barrier, and lateral-motion channel form a product space of potential
actions and penalties, but only some factors are active in a given
context.
We use \emph{action-space adaptation} in this restricted continuous
sense: the learner does not enumerate discrete actions, but changes which
force channels have non-negligible magnitude under the sensed context.

\paragraph{Static and dynamic risk.}
The same interface covers both static and dynamic settings.
In DFC and RELLIS, $\tilde r$ and $\phi$ are fixed semantic/material
fields, so risk exposure depends on the path taken through the map.
In highway-env, the local patch changes as neighboring vehicles move,
close gaps, or block lateral escape.
Static domains test context-enriched path selection; dynamic domains test
whether the learned field preserves clearance, time-to-collision, and
maneuvering room before violations become terminal.

\subsection{Route-aware gate: full specification}
\label{app:gate}

The main text uses the gate $m_{\rm feas}\in[0,1]$ in
Eq.~\eqref{eq:method:route_aware_force} without specifying its computation.
This subsection gives the full specification.

From the current state $q_t$ we roll out $K$ short-horizon motion
primitives $\{\xi_k\}_{k=1}^K$ inside the sensed BEV window $P_t$ and
compute for each candidate:
\begin{equation}
R_k = \int_{\xi_k}\tilde r(s)\,ds,
\qquad
C_k = \min_{\xi_k}\phi(s),
\qquad
I_k = \mathbf{1}\bigl[C_k>\delta_\phi
      \;\text{and}\;\xi_k\subset\mathcal{T}(P_t)\bigr],
\label{eq:gate_primitives}
\end{equation}
where $\mathcal{T}(P_t)$ is the traversable mask induced by the local
semantics.
Let $k^\star$ be the feasible candidate ($I_{k^\star}=1$) with lowest
soft-risk exposure, and let $k_{\rm scaf}$ be the candidate closest to
the geometry-only policy direction.
The affordance gate is
\begin{equation}
m_{\rm feas}(\vec c,z_t)
=
\sigma\!\bigl(\kappa_R[(R_{k_{\rm scaf}}-R_{k^\star})-\rho_R]\bigr)
\cdot
\sigma\!\bigl(\kappa_C(C_{k^\star}-\delta_\phi)\bigr)
\cdot
I_{k^\star}.
\label{eq:gate_full}
\end{equation}
Thus $m_{\rm feas}$ activates when: (i) a feasible candidate exists
($I_{k^\star}=1$), (ii) that candidate has clearance above threshold
($C_{k^\star}>\delta_\phi$), and (iii) its soft-risk integral is lower
than the geometry-only primitive by margin $\rho_R$.
When any condition fails, $m_{\rm feas}\approx0$ and $F_{\rm soft}$ is
suppressed without affecting $F_{\rm hard}$.

During one Hamiltonian integration substep $m_{\rm feas}$ is treated as
a locally measured context coefficient, exactly like the coefficients
predicted from the risk patch.
Equivalently, the substep uses local energy
$H_{\rm ctx}^{\rm ra}=m_{\rm feas}\lambda_s\tilde r+\lambda_h b_{\rm sp}(\phi)$,
with the gate coefficient frozen until the next BEV update.

\textbf{Implementation notes.}
The route-aware context-enriched field is not a full-map planner or a global-route oracle.
It uses only the local BEV patch, the local SDF, and finite-horizon
primitives within the same sensing radius available to all local-sensing
baselines.
The R1/R2/R3 regime labels are computed externally after the fact to
score whether activation was correct; they are not provided to the policy
at any time.
We use $K{=}8$, $\delta_\phi{=}0.15$\,m, $\rho_R{=}0.05$,
$\kappa_R{=}10$, $\kappa_C{=}20$ in all experiments.

\subsection{Coefficient prediction details}
\label{app:coeff_pred}

The context coefficients $(\lambda_s,\lambda_h)$ are predicted by
sigmoid-bounded heads from a CNN encoding of the risk patch $P_t$ fused
with the existing goal-context token:
\begin{equation}
z_r \leftarrow \psi_\omega(P_t),
\qquad
\lambda_s, \lambda_h \leftarrow \lambda^{\max}\,\sigma\!\bigl(\mathrm{MLP}(z_r,z_g)\bigr),
\end{equation}
where $z_g$ is the goal-context token from the geometry-only policy and $\lambda^{\max}$
is a fixed upper bound on each coefficient.
The geometric heads $(\alpha,\beta,\gamma)$ from the geometry-only policy are frozen,
making $(\lambda_s,\lambda_h)$ the only route by which context can alter
the decision field.
This isolation is what makes the constant-coefficient ablation
(Table~\ref{tab:rellis_selectivity}, constant-coeff row) interpretable:
setting $\lambda_s,\lambda_h$ to fixed scalars removes context-sensitivity
while keeping the force-channel structure, isolating the contribution of
adaptive coefficient prediction.

In highway-env we additionally expose a contextual lateral response
coefficient and a TTC-conditioned interaction channel; in boxed traffic
the TTC channel suppresses lateral activation while in the open
slow-leader case it allows activation.
A $\lambda$-entropy regularizer $L_\lambda$ prevents the sigmoid heads
from saturating.

\subsection{Balanced objective and evaluation metric mapping}
\label{app:jbal}

The full training objective combines the CVaR tail-risk term with a
balanced failure-profile cost:
\begin{equation}
\mathcal{J}_{\rm bal}(\pi)
=
\lambda_{\rm goal}C_{\rm goal}
+\lambda_{\rm risk}C_{\rm risk}
+\lambda_{\rm barrier}C_{\rm barrier}
+\lambda_{\rm control}C_{\rm control}
+\lambda_{\rm smooth}C_{\rm smooth}
+\lambda_{\rm dyn}C_{\rm dyn},
\label{eq:jbal}
\end{equation}
where the dynamic-risk term $C_{\rm dyn}$ is active for interaction
domains (highway-env) and zero otherwise.
The terms are not optimized independently: low risk by stopping, low
control by failing to evade, and short paths through hazards are all
failure modes.
The context-enriched field learns a context-conditioned compromise among these terms.

The one-to-one mapping from $\mathcal{J}_{\rm bal}$ terms to evaluation
metrics is deliberate: goal/progress $\to$ success, progress, path
length; soft risk $\to$ cumulative risk, soft CVaR; barriers $\to$
hard-hazard length, off-road events; control/smoothness $\to$ oscillation,
curvature, jerk; dynamic risk $\to$ clearance, TTC violation,
intervention window.
The metrics are not post-hoc decorations; they measure the same terms the
learned field is balancing.

\subsection{Why expected cost is insufficient}
\label{app:cvar_motivation}

Catastrophic contact, collision, or route failure is a tail event: the
bulk of rollouts may look acceptable while the failures that determine
deployment safety are rare.
Expected-cost training is then dominated by typical-case path quality
and allocates too little gradient to the rollouts where the new context
force channels must become active.

This is not merely a risk-aversion preference.
When training under expected cost, the gradient
$\partial\mathcal{L}/\partial\lambda_s$ receives contributions from all
rollouts, including the majority where $\nabla\tilde r$ is small or
where no risk-reduction opportunity exists.
Gradient signal from safe typical-case rollouts dilutes the signal from
the few rollouts where $-\tau\nabla\tilde r$ should drive a large
$\lambda_s$ update.
CVaR fixes this by restricting gradient to the worst $(1-\alpha)$
fraction: on those rollouts, $-\tau\nabla\tilde r$ contributes non-zero
$\partial q_{t+1}/\partial\lambda_s$ at every step, driving $\lambda_s$
toward values that bias the field away from high-risk regions exactly
where the safety signal lives.

The experimental evidence for this is in Table~\ref{tab:delayed_escape}:
context-enriched expected cost vs.\ route-aware CVaR.
Same architecture, same gate, different objective.
CVaR halves false pre-activation ($0.370\to0.180$), raises suppression
rate ($0.630\to0.820$), and improves navigation success by $30\%$
($0.620\to0.810$).
The R1-only subset (Table~\ref{tab:app_r1_only}) confirms the improvement
holds even when an escape genuinely exists, ruling out the explanation
that CVaR simply avoids acting.

\subsection{Context-enriched update loops and training algorithm}
\label{app:ctx_loops_hierarchy}

The context-enriched field's adaptation hierarchy comprises three concurrent loops that 
operate at different timescales, read different inputs, and update 
different objects.

\textbf{Loop 1 -- segment (within episode, deployment-time).}
A rank-1 secant Gauss--Newton corrector adjusts $(\beta,\gamma)$ and 
the top-$K$ nearest $\alpha_j$ between consecutive frames using realized 
clearance, goal distance, and speed. The Jacobian is maintained as 
$\hat J \leftarrow \mu\hat J + (1{-}\mu)\Delta y\Delta\zeta^\top/
\|\Delta\zeta\|^2$ and a damped least-squares step 
$\Delta\zeta^* = (\hat J^\top\hat J + \rho I)^{-1}\hat J^\top
(y_{\rm tgt}-y)$ is applied with per-coordinate rates and projection to 
$\mathbb{R}_{\ge0}$. This loop runs at every simulator step and uses 
no extra rollouts.

\textbf{Loop 2 -- episode (training-time, between rollouts).}
Backpropagation through the differentiable rollout integrator against 
Eq.~\eqref{eq:method:cvar_emp}, plus auxiliary losses: imitation to a risk-aware 
target ($L_{\rm traj}, L_{\rm vel}$), multi-start robustness 
($L_{\rm multi}$), clearance penalty ($L_{\rm clear}$), 
and $\lambda$-entropy regularizer ($L_\lambda$). Updates the full 
parameter vector including $\theta_R$ and risk encoder $\omega$.

\textbf{Loop 3 -- curriculum (training-time, between phases).}
Geometry-backbone weights are loaded from a geometry-only checkpoint with 
$\lambda_s, \lambda_h \equiv 0$ so $F_{\rm ctx} \equiv 0$ initially. 
Imitation weights are reduced to $0.3\times$ their geometry-only values. 
A passivity surrogate $\mathcal{L}_{\rm pass}$ on a held-out window 
gates phase advancement; the curriculum gradually exposes harder 
material patches, narrower corridors, stronger hazard penalties, and 
interaction scenarios.

\begin{figure}[h]
\centering
\includegraphics[width=\linewidth]{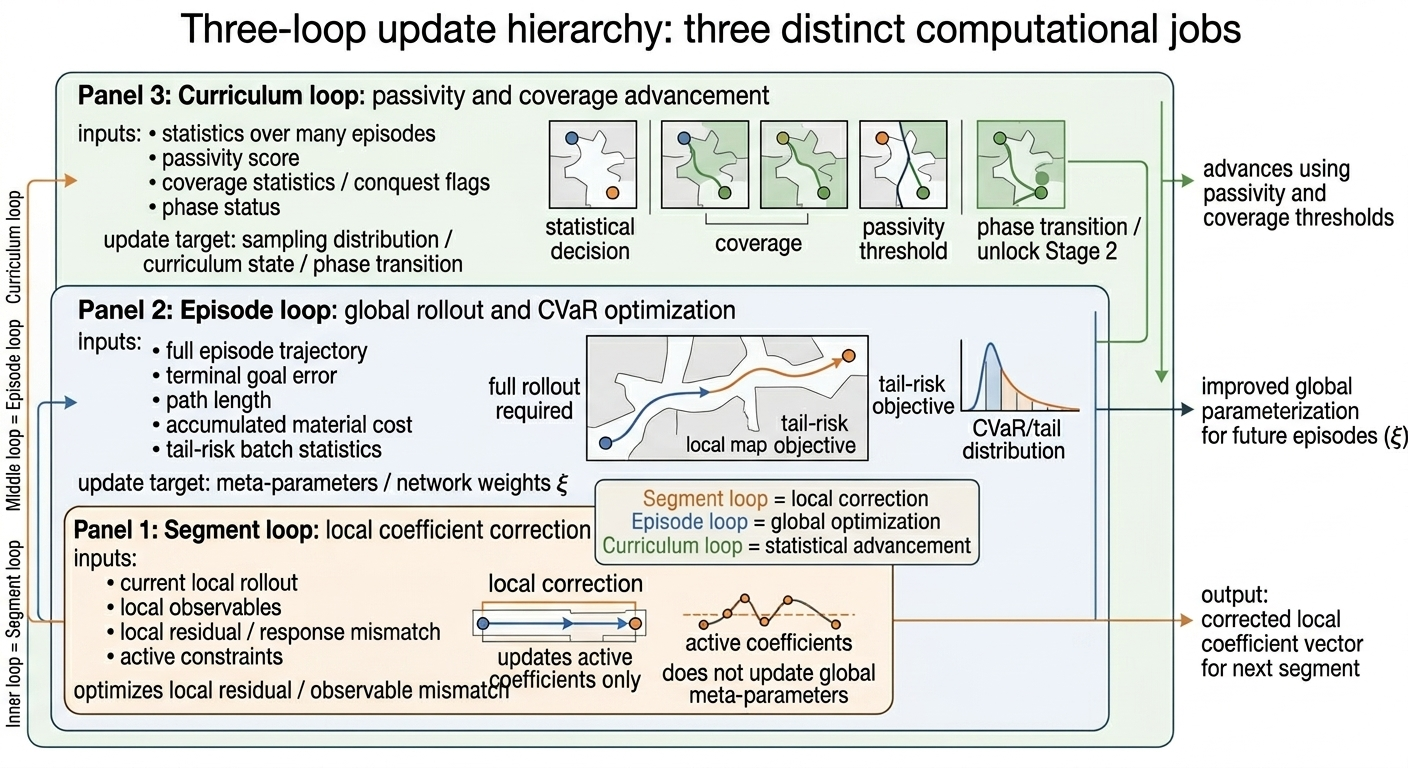}
\caption{\textbf{Three loops as three distinct computational jobs.}
The segment loop corrects active coefficients per step. The episode 
loop optimizes meta-parameters via CVaR. The curriculum loop 
advances training phases statistically.}
\label{fig:app_three_loops}
\end{figure}

\subsection{Context-enriched training algorithm}
\label{app:ctx_training_algorithm}

\begin{algorithm}[t]
\caption{Context-aware continual Hamiltonian learning (training; Loops 2 \& 3)}
\label{alg:ctx_train}
\begin{algorithmic}[1]
\Require geometry-only checkpoint $\theta_{\rm geo} = (\alpha,\beta,\gamma,\theta_R)$;
         dataset $\mathcal{D}$ of short rollouts with local context
         patches; CVaR level $\alpha_c{=}0.95$; batch size $B$;
         horizon $H$; step $\tau$.
\Statex \textbf{Loop 3 -- backward-compatible geometry-only $\to$ context-enriched initialization.}
\State $\theta_{\rm enr} \leftarrow \theta_{\rm scaf}$; initialize risk encoder
        $\omega$ and heads so $\lambda_s, \lambda_h \equiv 0$ at output
       \Comment{$F_{\mathrm{ctx}} \equiv 0$ initially}
\State Reduce imitation weights $(w_{\mathrm{traj}}, w_{\mathrm{vel}})
        \leftarrow 0.3\,(w_{\mathrm{traj}}, w_{\mathrm{vel}})$
\For{epoch $e = 1, \dots, E$}
  \Statex \hspace{0.7em}\textbf{Loop 2 -- episode loop (one minibatch).}
  \For{minibatch $\{(\Pos_0^{(i)}, \Mom_0^{(i)}, \Pos_g^{(i)},
                     \mathcal{O}^{(i)}, P^{(i)}, H^{(i)})\}_{i=1}^{B}
        \in \mathcal{D}$}
    \State $(\alpha^{(i)},\beta^{(i)},\gamma^{(i)})
            \leftarrow \phi_{\mathrm{geom}}(\mathcal{O}^{(i)},
                       \Pos_g^{(i)}{-}\Pos_0^{(i)})$
            \Comment{geometry transformer, unchanged from geometry-only policy}
    \State $z_r^{(i)} \leftarrow \psi_\omega(P^{(i)})$;\;
           $\lambda_s^{(i)},\lambda_h^{(i)} \leftarrow$
           $\lambda^{\max}\sigma(\text{MLP}(z_r^{(i)}, z_g^{(i)}))$
    \For{$t = 0, \dots, H^{(i)}{-}1$ (differentiable rollout)}
      \State $(\tilde r_t,\phi_t,\nabla\tilde r_t,\nabla\phi_t)
              \leftarrow \mathrm{BilinearSample}(P^{(i)},\Pos_t^{(i)})$
              \Comment{resample at \emph{current} $\Pos$, not $\Pos_0$}
      \State $m_{\mathrm{feas}}\leftarrow\mathrm{LocalGate}(P^{(i)},\Pos_t^{(i)},\Pos_g^{(i)})$
              \Comment{Eq.~\eqref{eq:method:route_gate}; set to $1$ for scalar context-enriched field}
      \State $F_{\mathrm{soft}} \leftarrow -m_{\mathrm{feas}}\lambda_s^{(i)} \nabla\tilde r_t$;\;
             $F_{\mathrm{hard}} \leftarrow -\lambda_h^{(i)}\,
                                          b_{\mathrm{sp}}'(\phi_t) \nabla\phi_t$
              \Comment{Eq.~\eqref{eq:method:route_aware_force}}
      \State $\Mom_{t+1}^{(i)} \leftarrow \Mom_t^{(i)}
              + \tau(F_{\mathrm{geom}} + F_{\mathrm{soft}}
                     + F_{\mathrm{hard}})
              - \tau\gamma^{(i)} \Mass^{-1}\Mom_t^{(i)}$
              \Comment{$R_\theta = \gamma\Mass^{-1}\Mom$;
                       $\Mass{=}I$ in implementation}
      \State $\Pos_{t+1}^{(i)} \leftarrow \Pos_t^{(i)} + \tau \Mass^{-1}\Mom_{t+1}^{(i)}$
      \State Accumulate $\mathrm{cum\_risk}, \mathrm{hard\_count}, \mathrm{arc}$
    \EndFor
    \State $J^{(i)} \leftarrow$ Eq.~\eqref{eq:method:J}
    \State $\hat\eta \leftarrow \widehat{Q}_{\alpha_c}(\{J^{(i)}\}_{i=1}^B)$
            \Comment{empirical quantile, \textbf{detach from graph}}
    \State $\widehat{\cvar{\alpha_c}} \leftarrow
           \hat\eta + \tfrac{1}{(1-\alpha_c)B}\sum_i (J^{(i)} - \hat\eta)_+$
            \Comment{Eq.~\eqref{eq:method:cvar_emp}; gradient flows
                     only through tail rollouts}
    \State $\mathcal{L} \leftarrow \widehat{\cvar{\alpha_c}}
                + \sum_{k} w_k L_k$
            \Comment{$L_k \in \{L_{\mathrm{traj}}, L_{\mathrm{vel}},
                                L_{\mathrm{fric}}, L_{\mathrm{clear}},
                                L_{\mathrm{multi}}, L_\lambda\}$}
    \State $\theta_{\rm enr} \leftarrow \theta_{\rm enr} - \mathrm{lr}\,
            \nabla_{\theta_{\rm enr}}\mathcal{L}$,
            $\;\|\nabla\|_2$ clipped to $5.0$
  \EndFor
  \Statex \hspace{0.7em}\textbf{Loop 3 -- curriculum advancement check.}
  \State Compute passivity surrogate $\mathcal{L}_{\mathrm{pass}}$ on
         validation window
  \If{violation rate $< \eta_{\mathrm{tol}}$ over last $W$ epochs}
    \State Mark phase as advanced; optionally widen risk-class
           distribution in $\mathcal{D}$
  \EndIf
\EndFor
\State \Return $\theta_{\rm enr}$
\end{algorithmic}
\end{algorithm}

\subsection{Energy enrichment induces force and sensitivity channels}

\begin{lemma}[Energy enrichment induces force, sensitivity, and excitation channels]
\label{lem:enrichment_channels}
Let
\[
    H_{\rm enr}(q,p;\theta,\vartheta)
    =
    H_{\rm scaf}(q,p;\theta)+\vartheta H_{\mathrm{new}}(q),
\]
where \(H_{\mathrm{new}}\in C^1\) and \(\vartheta\in\mathbb R\) is a learnable scalar.
Then:
\begin{enumerate}
    \item the enriched momentum equation contains the new force channel
    \[
        F_{\mathrm{new}}(q)=-\vartheta\nabla_q H_{\mathrm{new}}(q);
    \]
    \item differentiating a rollout-dependent loss \(J\) through the discrete update
    creates the direct sensitivity stream
    \[
        \frac{\partial J}{\partial \vartheta}
        =
        -\tau\sum_t
        \left\langle
        \frac{\partial J}{\partial p_{t+1}},
        \nabla_qH_{\mathrm{new}}(q_t)
        \right\rangle
        +
        \textnormal{recursive terms};
    \]
    \item local identifiability of \(\vartheta\) from force observations requires
    non-degenerate excitation of the new channel, i.e. the empirical Gram matrix
    over \(\{\nabla_qH_{\mathrm{new}}(q_t)\}_t\) must be nonsingular on the
    observed force subspace.
\end{enumerate}
\end{lemma}

\begin{proof}
Since
\[
    \nabla_q H_{\rm enr}
    =
    \nabla_q H_{\rm scaf}+\vartheta\nabla_qH_{\mathrm{new}},
\]
Hamilton's momentum equation contains
\[
    -\nabla_qH_{\rm enr}
    =
    -\nabla_qH_{\rm scaf}-\vartheta\nabla_qH_{\mathrm{new}}.
\]
Thus the added term contributes the force channel
\(F_{\mathrm{new}}=-\vartheta\nabla_qH_{\mathrm{new}}\).

For the sensitivity statement, consider the discrete momentum update
\[
    p_{t+1}
    =
    p_t-\tau\nabla_qH_{\rm scaf}(q_t,p_t)
    -\tau\vartheta\nabla_qH_{\mathrm{new}}(q_t)+\cdots .
\]
Differentiating with respect to \(\vartheta\) gives
\[
    \frac{\partial p_{t+1}}{\partial\vartheta}
    =
    \frac{\partial p_t}{\partial\vartheta}
    -\tau\nabla_qH_{\mathrm{new}}(q_t)
    -\tau\nabla_q^2H_{\rm enr}(q_t,p_t)
        \frac{\partial q_t}{\partial\vartheta}
    +\cdots .
\]
Applying the chain rule to any rollout loss \(J\) yields
\[
    \frac{\partial J}{\partial\vartheta}
    =
    \sum_t
    \left\langle
    \frac{\partial J}{\partial p_{t+1}},
    \frac{\partial p_{t+1}}{\partial\vartheta}
    \right\rangle
    +
    \sum_t
    \left\langle
    \frac{\partial J}{\partial q_{t+1}},
    \frac{\partial q_{t+1}}{\partial\vartheta}
    \right\rangle .
\]
The displayed direct stream is the term produced by
\(-\tau\nabla_qH_{\mathrm{new}}(q_t)\); the remaining terms are recursive
contributions through \(\partial q_t/\partial\vartheta\) and
\(\partial p_t/\partial\vartheta\).

For identifiability, suppose two coefficients \(\vartheta,\vartheta'\) induce the same
observed force on the rollout. Then
\[
    (\vartheta-\vartheta')\nabla_qH_{\mathrm{new}}(q_t)=0
\]
in all observed force directions. A nonzero coefficient difference is impossible only
when the observed gradients span the relevant force subspace. Equivalently, the
corresponding empirical Gram matrix is nonsingular on that subspace.
\end{proof}

\subsection{Force-dictionary separation}
\label{app:force_dictionary_separation}

\begin{lemma}[Context forces cannot be obtained by geometric reweighting]
Let the geometry-only force dictionary at \(q\) be
\[
    \mathcal D_{\mathrm{geom}}(q)
    =
    \operatorname{span}\{\nabla_q D_{\mathcal G}(q),
    \nabla_q b_{\mathrm{IPC}}(d_1(q)),\ldots,
    \nabla_q b_{\mathrm{IPC}}(d_N(q))\}.
\]
If \(P_{\mathcal D^\perp}\nabla_q\widetilde r(q)\neq0\), then there is no
choice of geometry-only coefficient reweighting that reproduces
\(F_{\mathrm{soft}}=-\lambda_s\nabla_q\widetilde r(q)\) for
\(\lambda_s\neq0\).
\end{lemma}

\begin{proof}
Any geometry-only coefficient reweighting remains in
\(\mathcal D_{\mathrm{geom}}(q)\). But
\(F_{\mathrm{soft}}\) has nonzero projection onto
\(\mathcal D_{\mathrm{geom}}(q)^\perp\) whenever
\(\lambda_sP_{\mathcal D^\perp}\nabla_q\widetilde r(q)\neq0\), so it cannot equal
any vector in the geometry-only span.
\end{proof}

\subsection{Passivity and practical discrete dissipativity}

\begin{theorem}[Contextual Hamiltonian enrichment preserves dissipativity]
\label{thm:continuous_passivity}
Under Assumptions~\ref{ass:regularity}--\ref{ass:diss}, if the context \(z\) is
fixed and \(u=0\), then along solutions of \eqref{eq:ct_q}--\eqref{eq:ct_p},
\[
    \frac{d}{dt}H_{\rm enr}(q(t),p(t);z)
    =
    -v(t)^\top D(q(t),p(t);z)v(t)
    \le
    -\rho\|v(t)\|^2
    \le 0 .
\]
Thus the addition of \(H_{\mathrm{ctx}}\) preserves the port-Hamiltonian dissipativity
structure. If \(z=z(t)\) and \(u\neq0\), then
\[
    \frac{d}{dt}H_{\rm enr}
    =
    -v^\top Dv
    +
    \left\langle \nabla_zH_{\mathrm{ctx}}(q;z),\dot z \right\rangle
    +
    \left\langle \nabla_pH_{\rm enr},Bu \right\rangle .
\]
Hence the enriched system is passive with storage \(H_{\rm enr}\) and supply rate
\[
    s(t)
    =
    \left\langle \nabla_zH_{\mathrm{ctx}}(q;z),\dot z \right\rangle
    +
    \left\langle \nabla_pH_{\rm enr},Bu \right\rangle .
\]
\end{theorem}

\begin{proof}
For fixed \(z\), differentiate \(H_{\rm enr}\) along trajectories:
\[
    \dot H_{\rm enr}
    =
    \langle\nabla_qH_{\rm enr},\dot q\rangle
    +
    \langle\nabla_pH_{\rm enr},\dot p\rangle .
\]
Using \(\dot q=\nabla_pH_{\rm enr}\) and
\(\dot p=-\nabla_qH_{\rm enr}-D\nabla_pH_{\rm enr}\), we obtain
\[
    \dot H_{\rm enr}
    =
    \langle\nabla_qH_{\rm enr},\nabla_pH_{\rm enr}\rangle
    +
    \left\langle
    \nabla_pH_{\rm enr},
    -\nabla_qH_{\rm enr}-D\nabla_pH_{\rm enr}
    \right\rangle .
\]
The two interconnection terms cancel:
\[
    \langle\nabla_qH_{\rm enr},\nabla_pH_{\rm enr}\rangle
    -
    \langle\nabla_pH_{\rm enr},\nabla_qH_{\rm enr}\rangle=0.
\]
Therefore
\[
    \dot H_{\rm enr}
    =
    -\langle\nabla_pH_{\rm enr},D\nabla_pH_{\rm enr}\rangle.
\]
Since \(\nabla_pH_{\rm enr}=M^{-1}p=v\),
\[
    \dot H_{\rm enr}=-v^\top Dv.
\]
By Assumption~\ref{ass:diss}, \(D\succeq \rho I\), so
\[
    -v^\top Dv\le -\rho\|v\|^2\le 0.
\]
If \(z=z(t)\), the chain rule adds
\[
    \langle\nabla_zH_{\rm enr},\dot z\rangle
    =
    \langle\nabla_zH_{\mathrm{ctx}},\dot z\rangle .
\]
If \(u\neq 0\), the input term \(Bu\) in \(\dot p\) contributes
\[
    \langle\nabla_pH_{\rm enr},Bu\rangle .
\]
Combining these terms gives the claimed supply-rate identity.
\end{proof}

\begin{theorem}[Practical discrete dissipativity and step-size condition]
\label{thm:discrete_passivity}
Under Assumptions~\ref{ass:regularity}--\ref{ass:diss}, let \(u_t=0\).
For the semi-implicit context-enriched update \eqref{eq:disc_p}--\eqref{eq:disc_q}, there
exists a finite constant \(C_{\mathcal K}>0\) such that
\[
    H_{\rm enr}(x_{t+1};z_t)-H_{\rm enr}(x_t;z_t)
    \le
    -\tau v_t^\top D_tv_t+C_{\mathcal K}\tau^2 .
\]
Consequently, the discrete rollout is practically dissipative. If
\[
    \tau < \frac{v_t^\top D_tv_t}{C_{\mathcal K}},
\]
then \(H_{\rm enr}(x_{t+1};z_t)<H_{\rm enr}(x_t;z_t)\). If
\(v_t^\top D_tv_t\ge d_{\min}>0\) uniformly on a region of interest, then any
\[
    \tau < \frac{d_{\min}}{C_{\mathcal K}}
\]
guarantees one-step energy decrease on that region.
\end{theorem}

\begin{proof}
Let \(f_z(x)\) be the continuous-time vector field defined by
\eqref{eq:ct_q}--\eqref{eq:ct_p} with \(u=0\). Since the semi-implicit update is a
first-order consistent discretization of this vector field and all derivatives are bounded
on compact \(\mathcal K\), there exists \(C_I>0\) such that
\[
    x_{t+1}=x_t+\tau f_{z_t}(x_t)+e_t,
    \qquad
    \|e_t\|\le C_I\tau^2 .
\]
Because \(H_{\rm enr}\in C^2(\mathcal K)\), its gradient is Lipschitz. Let \(L_H\) be a
Lipschitz constant for \(\nabla H_{\rm enr}\). The smoothness inequality gives
\[
    H_{\rm enr}(x_{t+1})-H_{\rm enr}(x_t)
    \le
    \langle\nabla H_{\rm enr}(x_t),x_{t+1}-x_t\rangle
    +
    \frac{L_H}{2}\|x_{t+1}-x_t\|^2 .
\]
Substituting \(x_{t+1}-x_t=\tau f_{z_t}(x_t)+e_t\),
\begin{align*}
    H_{\rm enr}(x_{t+1})-H_{\rm enr}(x_t)
    &\le
    \tau\langle\nabla H_{\rm enr}(x_t),f_{z_t}(x_t)\rangle
    +
    \langle\nabla H_{\rm enr}(x_t),e_t\rangle  \\
    &\qquad
    +
    \frac{L_H}{2}\|\tau f_{z_t}(x_t)+e_t\|^2 .
\end{align*}
By Theorem~\ref{thm:continuous_passivity},
\[
    \langle\nabla H_{\rm enr}(x_t),f_{z_t}(x_t)\rangle
    =
    -v_t^\top D_tv_t .
\]
The remaining two terms are bounded by \(C_{\mathcal K}\tau^2\) for some finite
constant \(C_{\mathcal K}\), since \(\mathcal K\) is compact and
\(\|\nabla H_{\rm enr}\|\), \(\|f_z\|\), and \(\|e_t\|/\tau^2\) are bounded there. Hence
\[
    H_{\rm enr}(x_{t+1};z_t)-H_{\rm enr}(x_t;z_t)
    \le
    -\tau v_t^\top D_tv_t+C_{\mathcal K}\tau^2 .
\]
Strict decrease follows whenever
\(C_{\mathcal K}\tau^2<\tau v_t^\top D_tv_t\). Dividing by \(\tau>0\) yields the
step-size condition. The uniform-margin result follows by replacing
\(v_t^\top D_tv_t\) by \(d_{\min}\).
\end{proof}

\subsection{Geometry-only preservation and no-hallucination}

\begin{theorem}[Geometry-only preservation under small context force]
\label{thm:scaffold_preservation}
Let \(x_t^{\rm geo}=(q_t^{\rm geo},p_t^{\rm geo})\) be the geometry-only rollout and
\(x_t^{\rm enr}=(q_t^{\rm enr},p_t^{\rm enr})\) be the context-enriched rollout, initialized from the same
state, and suppose Assumptions~\ref{ass:regularity},
\ref{ass:bounded_ctx}, and \ref{ass:lipschitz_rollout} hold. If along the rollout tube
\[
    \|\nabla_qH_{\mathrm{ctx}}(q;z)\|\le \varepsilon .
\]
Then after \(N\) steps,
\[
    \|x_N^{\rm enr}-x_N^{\rm geo}\|
    \le
    \tau\sqrt{1+\tau^2\|M^{-1}\|^2}\,
    \varepsilon\,
    \frac{(1+\tau L_\Phi)^N-1}{\tau L_\Phi},
\]
with the last factor interpreted as \(N\) when \(L_\Phi=0\). In particular, if
\(\varepsilon=0\), the context-enriched rollout recovers the geometry-only rollout up to numerical
roundoff.
\end{theorem}

\begin{proof}
The only difference between the geometry-only and context-enriched fields is the context term in the
momentum update. At one step, this perturbation is
\[
    \delta p_{t+1}
    =
    -\tau\nabla_qH_{\mathrm{ctx}}(q_t;z_t),
\]
and the induced position perturbation is
\[
    \delta q_{t+1}
    =
    \tau M^{-1}\delta p_{t+1}
    =
    -\tau^2M^{-1}\nabla_qH_{\mathrm{ctx}}(q_t;z_t).
\]
Thus
\[
    \|\delta p_{t+1}\|\le \tau\varepsilon,
    \qquad
    \|\delta q_{t+1}\|\le \tau^2\|M^{-1}\|\varepsilon.
\]
The one-step state perturbation is therefore bounded by
\[
    \|\delta x_{t+1}\|
    \le
    \tau\sqrt{1+\tau^2\|M^{-1}\|^2}\,\varepsilon .
\]
Let
\[
    e_t=\|x_t^{\rm enr}-x_t^{\rm geo}\|.
\]
By Assumption~\ref{ass:lipschitz_rollout},
\[
    e_{t+1}
    \le
    (1+\tau L_\Phi)e_t
    +
    \tau\sqrt{1+\tau^2\|M^{-1}\|^2}\,\varepsilon .
\]
Since \(e_0=0\), iterating gives
\[
    e_N
    \le
    \tau\sqrt{1+\tau^2\|M^{-1}\|^2}\,\varepsilon
    \sum_{j=0}^{N-1}(1+\tau L_\Phi)^j .
\]
Evaluating the geometric sum yields the result.
\end{proof}

\begin{corollary}[No hallucinated lateral escape]
\label{cor:no_hallucination}
Let \(P_\perp\) denote projection onto the lateral subspace. Suppose local
coordinates are baseline-aligned so that longitudinal context forces do not generate
lateral motion to first order. If
\[
    \|P_\perp\nabla_qH_{\mathrm{ctx}}(q;z)\|\le \varepsilon_\perp
\]
on the rollout tube, then
\[
    \|P_\perp(q_N^{\rm enr}-q_N^{\rm geo})\|
    \le
    C_{\perp,N}\varepsilon_\perp,
\]
where
\[
    C_{\perp,N}
    =
    \tau^2
    \|P_\perp M^{-1}P_\perp^\top\|
    \frac{(1+\tau L_\Phi)^N-1}{\tau L_\Phi}.
\]
When \(\varepsilon_\perp=0\), the context channel cannot produce a lateral escape.
\end{corollary}

\begin{proof}
Project the context-induced position perturbation:
\[
    P_\perp \delta q_{t+1}
    =
    -\tau^2P_\perp M^{-1}\nabla_qH_{\mathrm{ctx}}(q_t;z_t).
\]
Under baseline-aligned local coordinates, the first-order longitudinal--lateral cross
terms vanish, so
\[
    \|P_\perp \delta q_{t+1}\|
    \le
    \tau^2
    \|P_\perp M^{-1}P_\perp^\top\|
    \varepsilon_\perp .
\]
Propagating this one-step projected perturbation through the same Lipschitz recursion
as in Theorem~\ref{thm:scaffold_preservation} gives the claimed bound.
\end{proof}

\subsection{Selective risk deflection and risk reduction}

\begin{theorem}[Selective risk deflection and local risk reduction]
\label{thm:risk_deflection}
Let \(P_\perp\) project onto a feasible lateral subspace. Suppose that at a state \(q\),
\[
    \|P_\perp\nabla_q\widetilde r(\mathbf c;z)\|\ge \Delta,
\]
and suppose the projected hard-hazard force does not cancel the soft-risk direction by
more than \(\chi\):
\[
    \left\|
    P_\perp
    \left[
    \lambda_h b_{\mathrm{sp}}'(\phi(\mathbf c;z))\nabla_q\phi(\mathbf c;z)
    \right]
    \right\|
    \le \chi.
\]
Then
\[
    \|P_\perp F_{\mathrm{ctx}}(q;z)\|
    \ge
    \lambda_s\Delta-\chi .
\]
Thus, whenever \(\lambda_s\Delta>\chi\), the context channel is necessarily active in
the lateral subspace.

Assume further that \(\widetilde r\) has \(L_r\)-Lipschitz gradient,
\(M^{-1}\succeq m_M I\), and
\[
    \tau^2\lambda_s\|M^{-1}\|\le \frac{1}{L_r}.
\]
Then, ignoring shared geometry and dissipation terms up to \(O(\tau^3)\), the
one-step context-enriched perturbation satisfies
\[
    \widetilde r(q_{t+1}^{\rm enr})
    -
    \widetilde r(q_{t+1}^{\rm scaf})
    \le
    -\frac12\tau^2\lambda_s m_M \Delta^2
    +
    O(\tau^3).
\]
Over any active window \(\mathcal A\) on which the same margin condition holds,
\[
    \sum_{t\in\mathcal A}
    \left[
    \widetilde r(q_{t+1}^{\rm enr})
    -
    \widetilde r(q_{t+1}^{\rm scaf})
    \right]
    \le
    -\frac12|\mathcal A|\tau^2\lambda_s m_M \Delta^2
    +
    O(|\mathcal A|\tau^3).
\]
\end{theorem}

\begin{proof}
The context force is
\[
    F_{\mathrm{ctx}}
    =
    -\lambda_s\nabla_q\widetilde r
    -
    \lambda_h b_{\mathrm{sp}}'(\phi)\nabla_q\phi .
\]
Projecting onto the lateral subspace gives
\[
    P_\perp F_{\mathrm{ctx}}
    =
    -\lambda_sP_\perp\nabla_q\widetilde r
    -
    P_\perp\!\left[
    \lambda_h b_{\mathrm{sp}}'(\phi)\nabla_q\phi
    \right].
\]
By the reverse triangle inequality,
\[
    \|P_\perp F_{\mathrm{ctx}}\|
    \ge
    \lambda_s\|P_\perp\nabla_q\widetilde r\|
    -
    \left\|
    P_\perp\!\left[
    \lambda_h b_{\mathrm{sp}}'(\phi)\nabla_q\phi
    \right]
    \right\|.
\]
The assumptions imply
\[
    \|P_\perp F_{\mathrm{ctx}}\|\ge \lambda_s\Delta-\chi.
\]
This proves activation under \(\lambda_s\Delta>\chi\).

For the risk-reduction claim, isolate the soft-risk force
\[
    F_{\mathrm{soft}}=-\lambda_s\nabla_q\widetilde r(q).
\]
The position perturbation induced by this force in one semi-implicit step is
\[
    \delta q
    =
    -\tau^2\lambda_sM^{-1}\nabla_q\widetilde r(q).
\]
By \(L_r\)-smoothness,
\[
    \widetilde r(q+\delta q)
    \le
    \widetilde r(q)
    +
    \langle\nabla_q\widetilde r(q),\delta q\rangle
    +
    \frac{L_r}{2}\|\delta q\|^2 .
\]
The first-order term is
\[
    \langle\nabla_q\widetilde r(q),\delta q\rangle
    =
    -\tau^2\lambda_s
    \nabla_q\widetilde r(q)^\top
    M^{-1}
    \nabla_q\widetilde r(q).
\]
Since \(M^{-1}\succeq m_MI\) and
\(\|\nabla_q\widetilde r(q)\|\ge \Delta\) on the active projected direction,
\[
    \nabla_q\widetilde r(q)^\top M^{-1}\nabla_q\widetilde r(q)
    \ge
    m_M\Delta^2.
\]
Therefore
\[
    \langle\nabla_q\widetilde r(q),\delta q\rangle
    \le
    -\tau^2\lambda_s m_M\Delta^2.
\]
The second-order term is
\[
    \frac{L_r}{2}\|\delta q\|^2
    =
    \frac{L_r}{2}
    \tau^4\lambda_s^2
    \|M^{-1}\nabla_q\widetilde r(q)\|^2 .
\]
The step-size condition
\[
    \tau^2\lambda_s\|M^{-1}\|\le \frac{1}{L_r}
\]
ensures that this second-order term is controlled by the first-order decrease, yielding
\[
    \widetilde r(q+\delta q)-\widetilde r(q)
    \le
    -\frac12\tau^2\lambda_s m_M\Delta^2.
\]
The shared geometry and dissipative terms are identical in the geometry-only and context-enriched fields
local comparison; by smoothness, their contribution to the relative risk difference is
\(O(\tau^3)\). Summing the one-step inequality over \(\mathcal A\) proves the final
claim.
\end{proof}

\subsection{CVaR gradient consistency}

The episode cost used in the main text is
\[
    J(\theta)
    =
    w_g\|q_T-q_g\|^2
    +
    w_\ell\sum_t\|q_{t+1}-q_t\|
    +
    w_r\sum_t \widetilde r(q_t)\|q_{t+1}-q_t\|
    +
    w_h\sum_t \mathbf 1\{\phi(q_t)<\epsilon\}.
\]
The training objective uses the Rockafellar--Uryasev representation of CVaR:
\[
    \operatorname{CVaR}_\alpha(J_\theta)
    =
    \min_{\eta\in\mathbb R}
    \left[
    \eta+
    \frac{1}{1-\alpha}\mathbb E(J_\theta-\eta)_+
    \right].
\]

\begin{theorem}[Detached empirical CVaR gives a consistent tail-gradient estimator]
\label{thm:cvar_consistency}
Let \(J_\theta(\xi)\) be the rollout cost induced by parameter \(\theta\) and rollout
randomness \(\xi\). Fix \(\alpha\in(0,1)\). Assume:
\begin{enumerate}
    \item \(J_\theta(\xi)\) is differentiable in \(\theta\) almost surely;
    \item \(\|\nabla_\theta J_\theta(\xi)\|\le G\) almost surely;
    \item the distribution of \(J_\theta\) has a continuous density near its
    \(\alpha\)-quantile;
    \item \(\mathbb P[J_\theta=\eta_\alpha(\theta)]=0\), where
    \(\eta_\alpha(\theta)=\operatorname{VaR}_\alpha(J_\theta)\).
\end{enumerate}
Then
\[
    \nabla_\theta \operatorname{CVaR}_\alpha(J_\theta)
    =
    \frac{1}{1-\alpha}
    \mathbb E
    \left[
    \nabla_\theta J_\theta(\xi)\,
    \mathbf 1\{J_\theta(\xi)\ge \eta_\alpha(\theta)\}
    \right].
\]
Given \(B\) i.i.d. rollouts and the detached empirical quantile
\[
    \widehat \eta
    =
    \widehat Q_\alpha(\{J_\theta^{(i)}\}_{i=1}^B),
\]
define
\[
    \widehat g_B(\theta)
    =
    \frac{1}{(1-\alpha)B}
    \sum_{i=1}^B
    \nabla_\theta J_\theta^{(i)}
    \mathbf 1\{J_\theta^{(i)}\ge \widehat \eta\}.
\]
Then
\[
    \widehat g_B(\theta)
    \to
    \nabla_\theta \operatorname{CVaR}_\alpha(J_\theta)
\]
almost surely as \(B\to\infty\), and
\[
    \|\widehat g_B(\theta)
    -
    \nabla_\theta \operatorname{CVaR}_\alpha(J_\theta)\|
    =
    O_p(B^{-1/2}).
\]
If \(J_\theta\in[0,J_{\max}]\), then the empirical CVaR value obeys
\[
    |\widehat{\operatorname{CVaR}}_\alpha-\operatorname{CVaR}_\alpha|
    =
    O\!\left(
    \frac{J_{\max}}{1-\alpha}
    \sqrt{\frac{\log(1/\delta)}{B}}
    \right)
\]
with probability at least \(1-\delta\).
\end{theorem}

\begin{proof}
The Rockafellar--Uryasev variational form gives
\[
    \operatorname{CVaR}_\alpha(J_\theta)
    =
    \min_{\eta\in\mathbb R}
    \left[
    \eta+
    \frac{1}{1-\alpha}\mathbb E(J_\theta-\eta)_+
    \right].
\]
At the minimizer,
\[
    \eta^\star=\eta_\alpha(\theta).
\]
By the envelope theorem, when differentiating the optimized value with respect to
\(\theta\), the derivative of the optimizer \(\eta^\star(\theta)\) does not appear. Hence
\[
    \nabla_\theta \operatorname{CVaR}_\alpha(J_\theta)
    =
    \frac{1}{1-\alpha}
    \nabla_\theta
    \mathbb E[(J_\theta-\eta^\star)_+].
\]
Since \(J_\theta\) is differentiable almost surely and the quantile boundary has
probability zero,
\[
    \nabla_\theta(J_\theta-\eta^\star)_+
    =
    \nabla_\theta J_\theta\,
    \mathbf 1\{J_\theta\ge \eta^\star\}
\]
almost surely. The bounded-gradient assumption permits interchanging derivative and
expectation, proving
\[
    \nabla_\theta \operatorname{CVaR}_\alpha(J_\theta)
    =
    \frac{1}{1-\alpha}
    \mathbb E
    \left[
    \nabla_\theta J_\theta
    \mathbf 1\{J_\theta\ge \eta_\alpha\}
    \right].
\]

For the empirical estimator, continuity of the distribution near \(\eta_\alpha\) implies
strong consistency of the empirical quantile:
\[
    \widehat \eta \to \eta_\alpha
\]
almost surely. Since there is no atom at the quantile,
\[
    \mathbf 1\{J_\theta^{(i)}\ge\widehat\eta\}
    \to
    \mathbf 1\{J_\theta^{(i)}\ge\eta_\alpha\}
\]
almost surely. The summands are bounded by \(G/(1-\alpha)\), so dominated
convergence and the strong law of large numbers give
\[
    \widehat g_B(\theta)
    \to
    \frac{1}{1-\alpha}
    \mathbb E
    \left[
    \nabla_\theta J_\theta
    \mathbf 1\{J_\theta\ge\eta_\alpha\}
    \right]
    =
    \nabla_\theta \operatorname{CVaR}_\alpha(J_\theta)
\]
almost surely.

For the rate, decompose
\[
    \widehat g_B-g
    =
    A_B+R_B,
\]
where
\[
    A_B
    =
    \frac{1}{(1-\alpha)B}
    \sum_{i=1}^B
    \left[
    \nabla J_i\mathbf 1\{J_i\ge\eta_\alpha\}
    -
    \mathbb E(\nabla J\mathbf 1\{J\ge\eta_\alpha\})
    \right]
\]
and
\[
    R_B
    =
    \frac{1}{(1-\alpha)B}
    \sum_{i=1}^B
    \nabla J_i
    \left[
    \mathbf 1\{J_i\ge\widehat\eta\}
    -
    \mathbf 1\{J_i\ge\eta_\alpha\}
    \right].
\]
The first term is a bounded empirical average, so \(A_B=O_p(B^{-1/2})\). The
second term is nonzero only when \(J_i\) lies between \(\widehat\eta\) and
\(\eta_\alpha\). The probability mass of this interval is
\(O_p(|\widehat\eta-\eta_\alpha|)\), and standard empirical quantile theory gives
\(|\widehat\eta-\eta_\alpha|=O_p(B^{-1/2})\). Hence \(R_B=O_p(B^{-1/2})\), proving
the gradient rate.

For bounded \(J_\theta\), the empirical Rockafellar--Uryasev objective is a bounded
empirical process indexed by \(\eta\). Uniform concentration over
\(\eta\in[0,J_{\max}]\), followed by comparison at the empirical and population
minimizers, yields the displayed
\[
    O\!\left(
    \frac{J_{\max}}{1-\alpha}
    \sqrt{\frac{\log(1/\delta)}{B}}
    \right)
\]
concentration rate.
\end{proof}

\subsection{Risk-averse constrained OCP interpretation}
\label{app:risk_constrained_ocp}

This subsection formalizes the statement that the contextual Hamiltonian should be
viewed as an amortized approximation to a \emph{family} of risk-averse constrained
optimal control problems, not as an unconditional online optimizer. The Hamiltonian
policy is a structured feedback parameterization, while the constrained OCP is the
reference problem whose local Lagrangian force field it aims to approximate.

Consider a finite-horizon controlled system
\begin{equation}
    x_{t+1}=\Phi(x_t,u_t,z_t),
    \qquad x_t=(q_t,p_t),
    \qquad t=0,\dots,H-1,
    \label{eq:app:discrete_ocp_dynamics}
\end{equation}
with local context process $z_t$ and rollout randomness $\xi$. Let $J_0(\tau)$ denote
the nominal goal, path-length, smoothness, and control-effort cost. Let
$C_j(\tau)$ denote the exposure accumulated by risk channel $j$, for example
soft-material exposure, hard-hazard proximity, deformation risk, or interaction risk.
The risk-averse constrained reference family is
\begin{equation}
\begin{aligned}
    \min_{\pi\in\Pi}\quad
    & \mathbb E[J_0(\tau_\pi)] \\
    \mathrm{s.t.}\quad
    & \operatorname{CVaR}_{\alpha_j}(C_j(\tau_\pi))
      \le \bar\rho_j,
      \qquad j=1,\dots,m .
\end{aligned}
\label{eq:app:risk_constrained_ocp}
\end{equation}
Equivalently, for nonnegative multipliers $\lambda\in\mathbb R_+^m$, its Lagrangian
risk objective is
\begin{equation}
    \mathcal J_\lambda(\pi)
    =
    \mathbb E[J_0(\tau_\pi)]
    +
    \sum_{j=1}^m
    \lambda_j
    \bigl(
      \operatorname{CVaR}_{\alpha_j}(C_j(\tau_\pi))-\bar\rho_j
    \bigr).
    \label{eq:app:risk_lagrangian_ocp}
\end{equation}
Under a local constraint qualification, a locally optimal constrained policy admits
KKT multipliers~\citep{ShapiroDentchevaRuszczynski2009} $\lambda^\star$; the corresponding first-order conditions are those of
\eqref{eq:app:risk_lagrangian_ocp} at $\lambda^\star$ together with complementary
slackness. This is the precise sense in which the coefficients of risk-aware energy
terms can be interpreted as local risk prices.

The learned Hamiltonian policy class is
\begin{equation}
    H_\theta(q,p;z)
    =
    \frac12p^\top M^{-1}p
    +U_\theta(q,z),
    \qquad
    u_{\rm raw}=\pi_\theta(x,z),
    \label{eq:app:hamiltonian_policy_class}
\end{equation}
where the induced closed-loop vector field has the form
\begin{equation}
    F_\theta(x,z)
    =
    (J-R_\theta)\nabla_xH_\theta(x,z)+G(x)u_{\rm raw}.
\end{equation}
The next theorem states the approximation claim used in the main text. It is a local
inverse-optimality statement: if the Hamiltonian force approximates the local optimal
Lagrangian closed-loop force, then the resulting trajectories and risk-Lagrangian
values are close.

\begin{theorem}[Hamiltonian approximation of the risk-averse Lagrangian OCP]
\label{thm:hamiltonian_approx_risk_ocp}
Fix a compact rollout tube $\mathcal K$ and a multiplier vector
$\lambda\in\mathbb R_+^m$. Suppose that a locally optimal feedback
$\pi_\lambda^\star$ for \eqref{eq:app:risk_lagrangian_ocp} induces a closed-loop update
map $\Phi_\lambda^\star$ on $\mathcal K$, and the Hamiltonian feedback
$\pi_\theta$ induces a closed-loop update map $\Phi_\theta$. Assume:
\begin{enumerate}
    \item $\Phi_\lambda^\star$ is $L_\Phi$-Lipschitz on $\mathcal K$;
    \item the one-step model mismatch is uniformly bounded,
    \begin{equation}
        \sup_{x,z\in\mathcal K}
        \|\Phi_\theta(x,z)-\Phi_\lambda^\star(x,z)\|
        \le \varepsilon_\Phi ;
        \label{eq:app:one_step_phi_mismatch}
    \end{equation}
    \item the rollout costs $J_0,C_1,\dots,C_m$ are $L_J,L_{C_1},\dots,L_{C_m}$
    Lipschitz functions of the finite trajectory under the sup norm.
\end{enumerate}
Then trajectories initialized at the same $x_0$ satisfy
\begin{equation}
    \max_{0\le t\le H}\|x_t^\theta-x_t^\star\|
    \le
    C_H\varepsilon_\Phi,
    \qquad
    C_H:=\sum_{k=0}^{H-1}(1+L_\Phi)^k .
    \label{eq:app:ocp_traj_close}
\end{equation}
Moreover,
\begin{equation}
\begin{aligned}
    \mathcal J_\lambda(\pi_\theta)-\mathcal J_\lambda(\pi_\lambda^\star)
    \le
    \left(L_J+\sum_{j=1}^m\lambda_jL_{C_j}\right)
    C_H\varepsilon_\Phi .
\end{aligned}
\label{eq:app:risk_lagrangian_close}
\end{equation}
If the constrained problem \eqref{eq:app:risk_constrained_ocp} satisfies a local
constraint qualification and $\lambda^\star$ is a KKT multiplier at a local constrained
solution, then \eqref{eq:app:risk_lagrangian_close} applies to the local Lagrangian
relaxation of the constrained risk-averse OCP at $\lambda^\star$.
\end{theorem}

\begin{proof}
Let $e_t=x_t^\theta-x_t^\star$. By adding and subtracting
$\Phi_\lambda^\star(x_t^\theta,z_t)$,
\[
\begin{aligned}
    \|e_{t+1}\|
    &\le
    \|\Phi_\theta(x_t^\theta,z_t)-\Phi_\lambda^\star(x_t^\theta,z_t)\|
    +
    \|\Phi_\lambda^\star(x_t^\theta,z_t)-\Phi_\lambda^\star(x_t^\star,z_t)\| \\
    &\le
    \varepsilon_\Phi+(1+L_\Phi)\|e_t\| .
\end{aligned}
\]
Since $e_0=0$, discrete Gronwall gives
\[
    \max_{0\le t\le H}\|e_t\|
    \le
    \varepsilon_\Phi\sum_{k=0}^{H-1}(1+L_\Phi)^k
    =C_H\varepsilon_\Phi .
\]
The Lipschitz assumptions on the trajectory functionals imply
\[
    |J_0(\tau_\theta)-J_0(\tau_\star)|
    \le L_JC_H\varepsilon_\Phi,
    \qquad
    |C_j(\tau_\theta)-C_j(\tau_\star)|
    \le L_{C_j}C_H\varepsilon_\Phi .
\]
CVaR is monotone and translation equivariant. Therefore, if two random variables
$X,Y$ obey $|X-Y|\le a$ almost surely, then
$|\operatorname{CVaR}_\alpha(X)-\operatorname{CVaR}_\alpha(Y)|\le a$. Applying this
with $X=C_j(\tau_\theta)$, $Y=C_j(\tau_\star)$ and
$a=L_{C_j}C_H\varepsilon_\Phi$ gives
\[
    \left|
    \operatorname{CVaR}_{\alpha_j}(C_j(\tau_\theta))
    -
    \operatorname{CVaR}_{\alpha_j}(C_j(\tau_\star))
    \right|
    \le L_{C_j}C_H\varepsilon_\Phi .
\]
Combining the nominal and risk-channel bounds yields
\eqref{eq:app:risk_lagrangian_close}. Finally, under the stated constraint
qualification, the KKT conditions identify the local constrained solution with a
stationary point of the Lagrangian relaxation at some nonnegative multiplier
$\lambda^\star$ satisfying complementary slackness; substituting
$\lambda=\lambda^\star$ gives the constrained-OCP interpretation.
\end{proof}

\subsection{Identifiability, generalization, and upper-confidence risk safety}
\label{app:identifiability_generalization_safety}

This subsection makes precise what can and cannot be identified from learning the
risk-aware Hamiltonian, and how the empirical CVaR update becomes a genuine
risk-constraint update once an upper-confidence radius is added. The argument uses
three standard ingredients: uniform convergence for bounded function classes via
Rademacher complexity~\citep{bartlett2002rademacher,ShalevShwartzBenDavid2014},
quantile/CVaR concentration via the Dvoretzky--Kiefer--Wolfowitz inequality with
Massart's constant~\citep{DvoretzkyKieferWolfowitz1956,Massart1990}, and
forward-invariance of barrier-filtered dynamics via Nagumo/CBF conditions
\citep{Nagumo1942,ames2016control}.

\paragraph{Structured risk-energy class.}
For a local context patch $z$, write the learned potential as
\begin{equation}
    U_\theta(q,z)
    =
    U_{\rm goal}(q,z)
    +
    \sum_{i=1}^{N(z)} \alpha_i(z) B_i(q,z)
    +
    U_{{\rm res},\theta}(q,z),
    \label{eq:app:structured_energy}
\end{equation}
where $B_i$ are fixed differentiable obstacle or risk-barrier atoms and
$U_{{\rm res},\theta}$ is the learned residual shaping potential. The induced force is
\begin{equation}
    F_\theta(q,z)
    =
    -\nabla_q U_\theta(q,z).
\end{equation}
The scalar potential is never identifiable beyond an additive constant; only its
gradient can be identified from local force or direction observations. Moreover, the
split between the structured barrier part and the residual is identifiable only under an
excitation and non-absorption condition.

\begin{assumption}[Barrier excitation and residual non-absorption]
\label{ass:barrier_excitation}
For each context patch $z$, let
\[
    G_B(q,z)
    =
    [\nabla_q B_1(q,z),\dots,\nabla_q B_{N(z)}(q,z)]\in\mathbb R^{d\times N(z)}.
\]
There exists $\kappa_B>0$ such that
\begin{equation}
    \lambda_{\min}
    \left(
    \mathbb E_{q\sim\mu_z}[G_B(q,z)^\top G_B(q,z)]
    \right)
    \ge \kappa_B .
    \label{eq:app:barrier_gram}
\end{equation}
The residual force is either constrained to be small,
$\mathbb E_{\mu_z}\|\nabla_qU_{{\rm res},\theta}\|^2\le \varepsilon_{\rm res}^2$, or
orthogonal to the barrier dictionary,
\begin{equation}
    \mathbb E_{q\sim\mu_z}
    \left[G_B(q,z)^\top \nabla_qU_{{\rm res},\theta}(q,z)\right]=0 .
\end{equation}
\end{assumption}

\begin{proposition}[Local identifiability of the structured risk weights]
\label{prop:local_identifiability}
Fix $z$ and suppose the teacher force admits the decomposition
\begin{equation}
    F^\star(q,z)
    =
    -\nabla_q U_{\rm goal}(q,z)
    -G_B(q,z)\alpha^\star(z)
    +\epsilon(q,z),
    \label{eq:app:teacher_force_decomp}
\end{equation}
with $\mathbb E\|\epsilon(q,z)\|^2\le \sigma_\epsilon^2$. Let
$\widehat\alpha(z)$ minimize the population least-squares force error over the
structured class with residual either removed or constrained as in
Assumption~\ref{ass:barrier_excitation}. Then
\begin{equation}
    \|\widehat\alpha(z)-\alpha^\star(z)\|
    \le
    \frac{1}{\kappa_B}
    \left(
    \left\|
    \mathbb E[G_B(q,z)^\top\epsilon(q,z)]
    \right\|
    +
    \varepsilon_{\rm res}
    \sqrt{\mathbb E\|G_B(q,z)\|_{\rm op}^2}
    \right).
    \label{eq:app:alpha_identifiability_bound}
\end{equation}
In particular, if the model is well specified, the residual is orthogonal to the
barrier dictionary, and $\mathbb E[G_B^\top\epsilon]=0$, then
$\widehat\alpha(z)=\alpha^\star(z)$. Without the Gram condition
\eqref{eq:app:barrier_gram}, the structured weights are not identifiable.
\end{proposition}

\begin{proof}
Subtract $-\nabla U_{\rm goal}$ from both the teacher force and the model force. The
population normal equation for the structured coefficients is
\[
    \mathbb E[G_B^\top G_B]\widehat\alpha
    =
    \mathbb E[G_B^\top(G_B\alpha^\star-\epsilon-r_\theta)],
\]
where $r_\theta=\nabla_qU_{{\rm res},\theta}$, with the sign absorbed consistently in
both sides. Hence
\[
    \mathbb E[G_B^\top G_B](\widehat\alpha-\alpha^\star)
    =
    -\mathbb E[G_B^\top\epsilon]-\mathbb E[G_B^\top r_\theta].
\]
Taking norms and using
$\|A^{-1}\|_{\rm op}\le 1/\kappa_B$ gives
\[
    \|\widehat\alpha-\alpha^\star\|
    \le
    \frac{1}{\kappa_B}
    \left(
    \|\mathbb E[G_B^\top\epsilon]\|
    +
    \|\mathbb E[G_B^\top r_\theta]\|
    \right).
\]
If the residual is orthogonal to the dictionary, the second term vanishes. If it is only
small, Cauchy--Schwarz gives
\[
    \|\mathbb E[G_B^\top r_\theta]\|
    \le
    \sqrt{\mathbb E\|G_B\|_{\rm op}^2}
    \sqrt{\mathbb E\|r_\theta\|^2}
    \le
    \varepsilon_{\rm res}
    \sqrt{\mathbb E\|G_B\|_{\rm op}^2}.
\]
This proves the bound. If $\mathbb E[G_B^\top G_B]$ is singular, there exists a
nonzero vector $a$ in its nullspace. Then
$G_B(q,z)a=0$ for $\mu_z$-almost every $q$, so $\alpha$ and $\alpha+a$ induce the
same observed structured force. Thus the coefficients are not identifiable.
\end{proof}

\begin{assumption}[Bounded force class]
\label{ass:bounded_force_class}
Let
\[
    \mathcal F=\{(q,p,z)\mapsto F_\theta(q,p,z):\theta\in\Theta\}
\]
be the force class induced by the Hamiltonian network on the compact rollout tube
$\mathcal K$. Assume that
$\|F_\theta(q,p,z)-F^\star(q,p,z)\|\le M_F$ and that the squared loss class
\[
    \mathcal L_F
    =
    \{(q,p,z,F^\star)\mapsto \|F_\theta(q,p,z)-F^\star\|^2:\theta\in\Theta\}
\]
is measurable and has empirical Rademacher complexity
$\mathfrak R_n(\mathcal L_F)$.
\end{assumption}

\begin{theorem}[Force-field generalization]
\label{thm:force_generalization}
Let $(q_i,p_i,z_i,F_i^\star)_{i=1}^n$ be i.i.d. samples from the local patch
distribution, and define
\[
    \mathcal R_F(\theta)
    =
    \mathbb E\|F_\theta(q,p,z)-F^\star(q,p,z)\|^2,
    \qquad
    \widehat{\mathcal R}_{F,n}(\theta)
    =
    \frac1n\sum_{i=1}^n
    \|F_\theta(q_i,p_i,z_i)-F_i^\star\|^2 .
\]
Under Assumption~\ref{ass:bounded_force_class}, with probability at least
$1-\delta$,
\begin{equation}
    \mathcal R_F(\hat\theta)
    \le
    \widehat{\mathcal R}_{F,n}(\hat\theta)
    +
    2\mathfrak R_n(\mathcal L_F)
    +
    3M_F^2\sqrt{\frac{\log(2/\delta)}{2n}}
    \label{eq:app:force_gen_bound}
\end{equation}
for every data-dependent estimator $\hat\theta\in\Theta$. If the learned dynamics and
costs are Lipschitz on $\mathcal K$, this force error propagates to rollout-cost error
with a finite-horizon Gronwall factor.
\end{theorem}

\begin{proof}
The first statement is the standard Rademacher uniform convergence inequality for a
bounded loss class~\citep{bartlett2002rademacher,ShalevShwartzBenDavid2014}:
with probability at least $1-\delta$,
\[
    \sup_{\theta\in\Theta}
    \bigl(\mathcal R_F(\theta)-\widehat{\mathcal R}_{F,n}(\theta)\bigr)
    \le
    2\mathfrak R_n(\mathcal L_F)
    +
    3M_F^2\sqrt{\frac{\log(2/\delta)}{2n}} .
\]
Substituting the random estimator $\hat\theta$ into this uniform event gives
\eqref{eq:app:force_gen_bound}. For the rollout statement, let
$e_t=x_t^\theta-x_t^\star$. If the closed-loop update map is $L_\Phi$-Lipschitz and
the one-step force mismatch is bounded by $\varepsilon_F$, then
\[
    \|e_{t+1}\|
    \le
    (1+\tau L_\Phi)\|e_t\|+\tau\varepsilon_F .
\]
Iterating gives
\[
    \max_{0\le t\le H}\|e_t\|
    \le
    \tau\varepsilon_F
    \sum_{k=0}^{H-1}(1+\tau L_\Phi)^k .
\]
Lipschitz continuity of the finite-horizon cost then converts the trajectory bound into
an objective bound.
\end{proof}

\begin{lemma}[Upper-confidence CVaR bound]
\label{lem:cvar_ucb}
Let $C(\tau_\theta)\in[0,C_{\max}]$ be a bounded rollout risk exposure and let
$C_1,\dots,C_B$ be i.i.d. validation rollouts for a fixed policy $\theta$. Define the
population and empirical Rockafellar--Uryasev objectives
\begin{align}
    \varphi(\eta)
    &=
    \eta+\frac{1}{1-\alpha}\mathbb E[(C-\eta)_+], \\
    \widehat\varphi_B(\eta)
    &=
    \eta+\frac{1}{(1-\alpha)B}\sum_{i=1}^B(C_i-\eta)_+ .
\end{align}
Let
\[
    \rho_\alpha=\inf_{\eta\in\mathbb R}\varphi(\eta),
    \qquad
    \widehat\rho_{\alpha,B}=\inf_{\eta\in\mathbb R}\widehat\varphi_B(\eta).
\]
Then, with probability at least $1-\delta$,
\begin{equation}
    \rho_\alpha
    \le
    \widehat\rho_{\alpha,B}
    +
    \frac{C_{\max}}{1-\alpha}
    \sqrt{\frac{\log(2/\delta)}{2B}} .
    \label{eq:app:cvar_ucb_single}
\end{equation}
For $m$ bounded risk channels $C_j\in[0,C_{j,\max}]$, the bound holds
simultaneously for all channels with probability at least $1-\delta$ after using
\begin{equation}
    \beta_{j,B}(\delta/m)
    :=
    \frac{C_{j,\max}}{1-\alpha_j}
    \sqrt{\frac{\log(2m/\delta)}{2B}} .
    \label{eq:app:explicit_cvar_beta}
\end{equation}
\end{lemma}

\begin{proof}
Because $C\in[0,C_{\max}]$, every minimizer of $\varphi$ and
$\widehat\varphi_B$ may be chosen in $[0,C_{\max}]$: if $\eta<0$, increasing
$\eta$ to $0$ cannot increase the Rockafellar--Uryasev objective, and if
$\eta>C_{\max}$, decreasing $\eta$ to $C_{\max}$ cannot increase it. Hence it is
enough to control the objectives uniformly over $\eta\in[0,C_{\max}]$.

Let $F$ and $F_B$ denote the population and empirical distribution functions of
$C$. For $\eta\in[0,C_{\max}]$, the tail-integral identity gives
\begin{equation}
    \mathbb E[(C-\eta)_+]
    =
    \int_\eta^{C_{\max}}(1-F(t))\,dt,
    \qquad
    \frac1B\sum_{i=1}^B(C_i-\eta)_+
    =
    \int_\eta^{C_{\max}}(1-F_B(t))\,dt .
    \label{eq:app:tail_integral_identity}
\end{equation}
Therefore, for every $\eta\in[0,C_{\max}]$,
\begin{align}
    |\widehat\varphi_B(\eta)-\varphi(\eta)|
    &\le
    \frac{1}{1-\alpha}
    \int_\eta^{C_{\max}}|F_B(t)-F(t)|\,dt \\
    &\le
    \frac{C_{\max}}{1-\alpha}
    \|F_B-F\|_\infty .
    \label{eq:app:ru_uniform_from_dkw}
\end{align}
Massart's sharp form of the Dvoretzky--Kiefer--Wolfowitz inequality states that
\citep{DvoretzkyKieferWolfowitz1956,Massart1990}
\begin{equation}
    \mathbb P\!\left(\|F_B-F\|_\infty>\epsilon\right)
    \le
    2\exp(-2B\epsilon^2).
\end{equation}
Thus, with probability at least $1-\delta$,
\[
    \|F_B-F\|_\infty
    \le
    \sqrt{\frac{\log(2/\delta)}{2B}} .
\]
Combining this event with \eqref{eq:app:ru_uniform_from_dkw} gives
\begin{equation}
    \sup_{\eta\in[0,C_{\max}]}
    |\widehat\varphi_B(\eta)-\varphi(\eta)|
    \le
    \frac{C_{\max}}{1-\alpha}
    \sqrt{\frac{\log(2/\delta)}{2B}} .
    \label{eq:app:ru_uniform_event}
\end{equation}
Finally, for any two functions $f,g$, $|\inf f-\inf g|\le\sup|f-g|$. Applying this
with $f=\varphi$ and $g=\widehat\varphi_B$ on the event
\eqref{eq:app:ru_uniform_event} yields
\[
    \rho_\alpha
    \le
    \widehat\rho_{\alpha,B}
    +
    \frac{C_{\max}}{1-\alpha}
    \sqrt{\frac{\log(2/\delta)}{2B}} .
\]
The simultaneous result follows by applying the same argument to each channel with
failure probability $\delta/m$ and taking a union bound.
\end{proof}

\begin{theorem}[Upper-confidence risk-constrained learning]
\label{thm:ucb_risk_constrained_learning}
For each risk channel $j=1,\dots,m$, let
\[
    \rho_j(\theta)
    =
    \operatorname{CVaR}_{\alpha_j}(C_j(\tau_\theta)),
    \qquad
    \widehat\rho_{j,B}(\theta)
    =
    \widehat{\operatorname{CVaR}}_{\alpha_j,B}(C_j(\tau_\theta))
    +
    \beta_{j,B}(\delta/m),
\]
where $\beta_{j,B}$ is chosen as in Eq.~\eqref{eq:app:explicit_cvar_beta}. Let
$\hat\theta$ be fixed before drawing the validation rollouts used to compute
$\widehat\rho_{j,B}$; alternatively, let $\beta_{j,B}$ include a uniform
function-class complexity term. If
\begin{equation}
    \widehat\rho_{j,B}(\hat\theta)\le \bar\rho_j,
    \qquad j=1,\dots,m,
    \label{eq:app:empirical_ucb_feasible}
\end{equation}
then, with probability at least $1-\delta$,
\begin{equation}
    \rho_j(\hat\theta)
    \le
    \bar\rho_j,
    \qquad j=1,\dots,m .
    \label{eq:app:true_risk_feasible}
\end{equation}
Moreover, the projected dual update
\begin{equation}
    \lambda_j^{k+1}
    =
    \left[
    \lambda_j^k
    +
    \eta_\lambda
    \bigl(\widehat\rho_{j,B}(\theta_k)-\bar\rho_j\bigr)
    \right]_+
    \label{eq:app:dual_update}
\end{equation}
performs ascent on the empirical upper-bound Lagrangian constraint violation and
therefore increases the risk price exactly when the upper-confidence budget is
violated.
\end{theorem}

\begin{proof}
Because the validation rollouts are independent of the fixed iterate $\hat\theta$,
Lemma~\ref{lem:cvar_ucb} applies to that iterate. A union bound over the $m$ risk
channels implies that, with probability at least $1-\delta$,
\[
    \rho_j(\hat\theta)
    \le
    \widehat\rho_{j,B}(\hat\theta)
\]
holds simultaneously for all $j$. Combining this event with
\eqref{eq:app:empirical_ucb_feasible} gives
\[
    \rho_j(\hat\theta)
    \le
    \widehat\rho_{j,B}(\hat\theta)
    \le
    \bar\rho_j,
\]
which proves true risk feasibility. The dual statement follows because the Lagrangian
term for channel $j$ is
$\lambda_j(\widehat\rho_{j,B}(\theta)-\bar\rho_j)$ with the constraint
$\lambda_j\ge 0$; projected gradient ascent in $\lambda_j$ is exactly
\eqref{eq:app:dual_update}.
\end{proof}

\begin{theorem}[Optional CBF projection condition]
\label{thm:risk_filter_invariance}
Let
\[
    \mathcal S
    =
    \{x: \bar R_j(x,z)\le \bar\rho_j,\; j=1,\dots,m\}
\]
where each $\bar R_j$ is continuously differentiable in $x$ for fixed $z$. Suppose
that, at every $x\in\mathcal S$, the online filter chooses a feasible input $u$ satisfying
\begin{equation}
    \nabla_x\bar R_j(x,z)^\top(f(x)+G(x)u)
    \le
    -\kappa_j(\bar R_j(x,z)-\bar\rho_j),
    \qquad j=1,\dots,m,
    \label{eq:app:risk_cbf_condition}
\end{equation}
with $\kappa_j>0$. Then $\mathcal S$ is forward invariant for the filtered continuous-time
closed-loop dynamics. If the optional port-power constraint $y^\top u\le \delta$ is also
included and the filter remains feasible, the same invariance conclusion holds with the
additional passivity budget enforced pointwise.
\end{theorem}

\begin{proof}
Define $h_j(x,z)=\bar\rho_j-\bar R_j(x,z)$. The safe set is
$\mathcal S=\{x:h_j(x,z)\ge0,\ j=1,\dots,m\}$. Condition
\eqref{eq:app:risk_cbf_condition} is equivalent to
\[
    \dot h_j(x,z)
    =
    -\nabla_x\bar R_j(x,z)^\top(f(x)+G(x)u)
    \ge
    \kappa_j(\bar R_j(x,z)-\bar\rho_j)
    =
    -\kappa_j h_j(x,z).
\]
By the comparison lemma, if $h_j(0)\ge0$, then
$h_j(t)\ge e^{-\kappa_jt}h_j(0)\ge0$ for all times on which the solution exists.
Equivalently, no trajectory starting in $\mathcal S$ can cross a boundary
$h_j=0$ outward. This is the standard Nagumo/CBF forward-invariance argument
\citep{Nagumo1942,ames2016control}. Adding $y^\top u\le\delta$ only further restricts
the feasible input set; if feasibility is preserved, the already-imposed inequalities
continue to imply invariance.
\end{proof}

\subsection{Barrier caveat: finite barriers are not CBF certificates}
\label{app:safety_caveat}

The hard-hazard term \(b_{\mathrm{sp}}(\phi)\) is a differentiable repulsive barrier. It
should not be stated as a formal forward-invariance certificate unless an additional
condition is imposed. In continuous time, forward invariance follows if the barrier
diverges at the unsafe boundary and total energy is bounded. In discrete time, one
needs either a sufficiently small step size that prevents crossing between samples, or
an explicit control-barrier-function projection/filter satisfying a discrete or
continuous CBF condition~\citep{ames2016control}. The optional projection theorem above
states such a sufficient condition, but the empirical Stage~2 results in this paper use
the Hamiltonian risk field without claiming a formal CBF certificate. Therefore, our
main claims are geometry-only preservation, dissipativity, selectivity, and local empirical
risk reduction under stated regularity assumptions; they are not unconditional
collision-avoidance guarantees.

\subsection{Empirical checks for the selectivity consequences}
\label{app:theory_empirical_checks}

The theory statements above are used as falsifiable rollout predictions, not as
unconditional safety certificates. Table~\ref{tab:app_theory_checks} states the
condition, the corresponding empirical measurement, and where it is tested.
The preservation bound in Theorem~\ref{thm:scaffold_preservation}
predicts that small context force yields small deviation from the geometry-only rollout.
Table~\ref{tab:app_scaffold_bound_check} reports the finite-rollout version of
this check at the same deviation threshold used in the experiments.

\begin{table}[h]
\centering
\caption{\textbf{Theory-to-metric checks.} Each theoretical consequence is
paired with an empirical statistic in the main experiments. The paper reports
these as measured rollout behavior, not as a formal invariance certificate.}
\label{tab:app_theory_checks}
\small
\begin{tabular}{p{0.23\linewidth}p{0.32\linewidth}p{0.34\linewidth}}
\toprule
Consequence & Checkable prediction & Empirical test \\
\midrule
C1: geometry-only preservation
& If $\|F_{\rm ctx}\|$ is small, trajectory deviation from the geometry-only rollout is small.
& R3 preservation / low pre-escape deviation; measured through false
pre-activation and suppression rate in Table~\ref{tab:delayed_escape}. \\
C2: no hallucinated escape
& If $P_\perp F_{\rm ctx}$ is small in blocked regimes, lateral escape attempts
are rare.
& RELLIS static FAR and delayed-required false pre-activation
(Tables~\ref{tab:rellis_selectivity}, \ref{tab:delayed_escape}). \\
C3: risk deflection
& If a feasible lower-risk direction has positive projected margin, the context
force aligns with it and risk exposure decreases after activation.
& RELLIS static CAR/SR/AUPRC and delayed-required success/CVaR
(Tables~\ref{tab:rellis_selectivity}, \ref{tab:delayed_escape}). \\
\bottomrule
\end{tabular}
\end{table}

\begin{table}[h]
\centering
\caption{\textbf{Empirical geometry-only deviation check for Theorem~\ref{thm:scaffold_preservation}.}
The theorem gives $\|P_\perp(q_t-q_t^{\rm geo})\|\le C_{\perp,N}\varepsilon_\perp$.
The experiments instantiate this as a thresholded rollout check: when the
projected context force is suppressed, lateral deviation from the geometry-only rollout
should stay below the measured reaction threshold $\delta$.}
\label{tab:app_scaffold_bound_check}
\small
\setlength{\tabcolsep}{4pt}
\renewcommand{\arraystretch}{1.08}
\resizebox{\linewidth}{!}{%
\begin{tabular}{lccc}
\toprule
Rollout condition & Bound prediction & Observed below-$\delta$ & Failure metric \\
\midrule
Geometry-only policy, delayed escape
& $\varepsilon_\perp=0 \Rightarrow$ no context deviation
& $1.000$ & false pre-act $0.000$ \\
Route-aware Ctx, static R2
& suppressed $P_\perp F_{\rm ctx}$ $\Rightarrow$ no hallucinated escape
& $0.886$ & FAR $0.114$ \\
Route-aware Ctx CVaR, pre-escape delayed
& pre-escape suppression $\Rightarrow$ no early lateral deviation
& $0.820$ & false pre-act $0.180$ \\
\bottomrule
\end{tabular}}
\end{table}

\section{Additional Experimental Results}
\label{app:additional_experiments}

This appendix preserves the broader diagnostic tables and qualitative
panels while keeping the main paper focused on the selectivity
signature.

\subsection{DFC2018 Full Planner Comparison}
\label{app:dfc_full}

Table~\ref{tab:app_dfc_main_metrics} reports the full DFC2018
comparison including privileged full-map planners.
The interpretation is unchanged from the main text: full-map planners
with global access obtain lower raw static risk, but context-enriched field
substantially repairs the geometry-only policy while producing smoother
continuous trajectories and zero oscillation.
PPO-Lagrangian fails to find a valid path on $20\%$ of episodes.

\begin{table}[h]
\centering
\caption{\textbf{DFC2018 full evaluation.}
Means over $300$ paired test episodes.
Oracle A$^*$ is privileged (global map access) and serves as a
reference ceiling, not a fair local-sensing competitor.}
\label{tab:app_dfc_main_metrics}
\scriptsize
\setlength{\tabcolsep}{3pt}
\renewcommand{\arraystretch}{1.08}
\resizebox{\linewidth}{!}{%
\begin{tabular}{lrrrrrrrr}
\toprule
Method & Success$\uparrow$ & Cat.\ fail$\downarrow$ & Hard len.$\downarrow$
       & Risk$\downarrow$ & Risk/m$\downarrow$ & Len.\ ratio
       & Osc.$\downarrow$ & Fail score$\downarrow$ \\
\midrule
Blind Dijkstra            & $1.000$ & $0.967$ & $55.386$ & $55.292$ & $0.244$ & $1.000$ & $1.990$  & $12927.5$ \\
Geometry-only A$^*$       & $1.000$ & $0.033$ & $0.576$  & $10.365$ & $0.042$ & $1.069$ & $24.819$ & $164.5$ \\
Risk-weighted A$^*$       & $1.000$ & $0.033$ & $0.333$  & $3.753$  & $0.014$ & $1.106$ & $54.480$ & $127.0$ \\
Oracle A$^*$$^\dagger$    & $1.000$ & $0.000$ & $0.000$  & $3.664$  & $0.014$ & $1.169$ & $58.512$ & $37.2$ \\
CVaR costmap A$^*$        & $1.000$ & $0.033$ & $0.287$  & $3.318$  & $0.012$ & $1.143$ & $72.780$ & $120.8$ \\
Chance-constrained MPC    & $1.000$ & $0.033$ & $0.287$  & $4.171$  & $0.016$ & $1.116$ & $58.957$ & $114.0$ \\
PPO-Lagrangian            & $0.800$ & $0.233$ & $0.355$  & $3.753$  & $0.014$ & $1.104$ & $13.221$ & $200.3$ \\
Ours geometry-only policy        & $0.867$ & $0.600$ & $21.106$ & $23.267$ & $0.099$ & $1.044$ & $82.212$ & $53216.6$ \\
\textbf{Ours ctx-enriched}     & $\mathbf{1.000}$ & $\mathbf{0.100}$ & $\mathbf{0.658}$
                          & $\mathbf{9.493}$ & $\mathbf{0.039}$ & $\mathbf{1.045}$
                          & $\mathbf{7.629}$ & $\mathbf{115.8}$ \\
\bottomrule
\end{tabular}}
\end{table}

\begin{table}[h]
\centering
\caption{\textbf{DFC failure-mode breakdown.}
Episode fractions.
Risk planners reduce raw material cost but frequently introduce
discrete oscillation; the context-enriched field primarily repairs the learned geometry
policy.}
\label{tab:app_dfc_failure_modes}
\scriptsize
\setlength{\tabcolsep}{4pt}
\renewcommand{\arraystretch}{1.08}
\begin{tabular}{lrrrrr}
\toprule
Method & No path & Hard haz. & High soft risk & Excess detour & Osc. \\
\midrule
Blind Dijkstra         & $0.00$ & $0.97$ & $0.97$ & $0.00$ & $0.00$ \\
Geometry A$^*$         & $0.00$ & $0.03$ & $0.20$ & $0.07$ & $0.43$ \\
Risk-weighted A$^*$    & $0.00$ & $0.03$ & $0.07$ & $0.10$ & $0.83$ \\
CVaR costmap A$^*$     & $0.00$ & $0.03$ & $0.07$ & $0.17$ & $0.90$ \\
Chance-const.\ MPC     & $0.00$ & $0.03$ & $0.07$ & $0.07$ & $0.83$ \\
PPO-Lagrangian         & $0.20$ & $0.23$ & $0.10$ & $0.13$ & $0.23$ \\
Ours geometry-only policy     & $0.13$ & $0.60$ & $0.63$ & $0.00$ & $0.63$ \\
\textbf{Ours ctx-enriched}  & $\mathbf{0.00}$ & $\mathbf{0.10}$ & $\mathbf{0.40}$
                       & $\mathbf{0.03}$ & $\mathbf{0.00}$ \\
\bottomrule
\end{tabular}
\end{table}

\begin{figure*}[p]
\centering
\includegraphics[width=\linewidth]{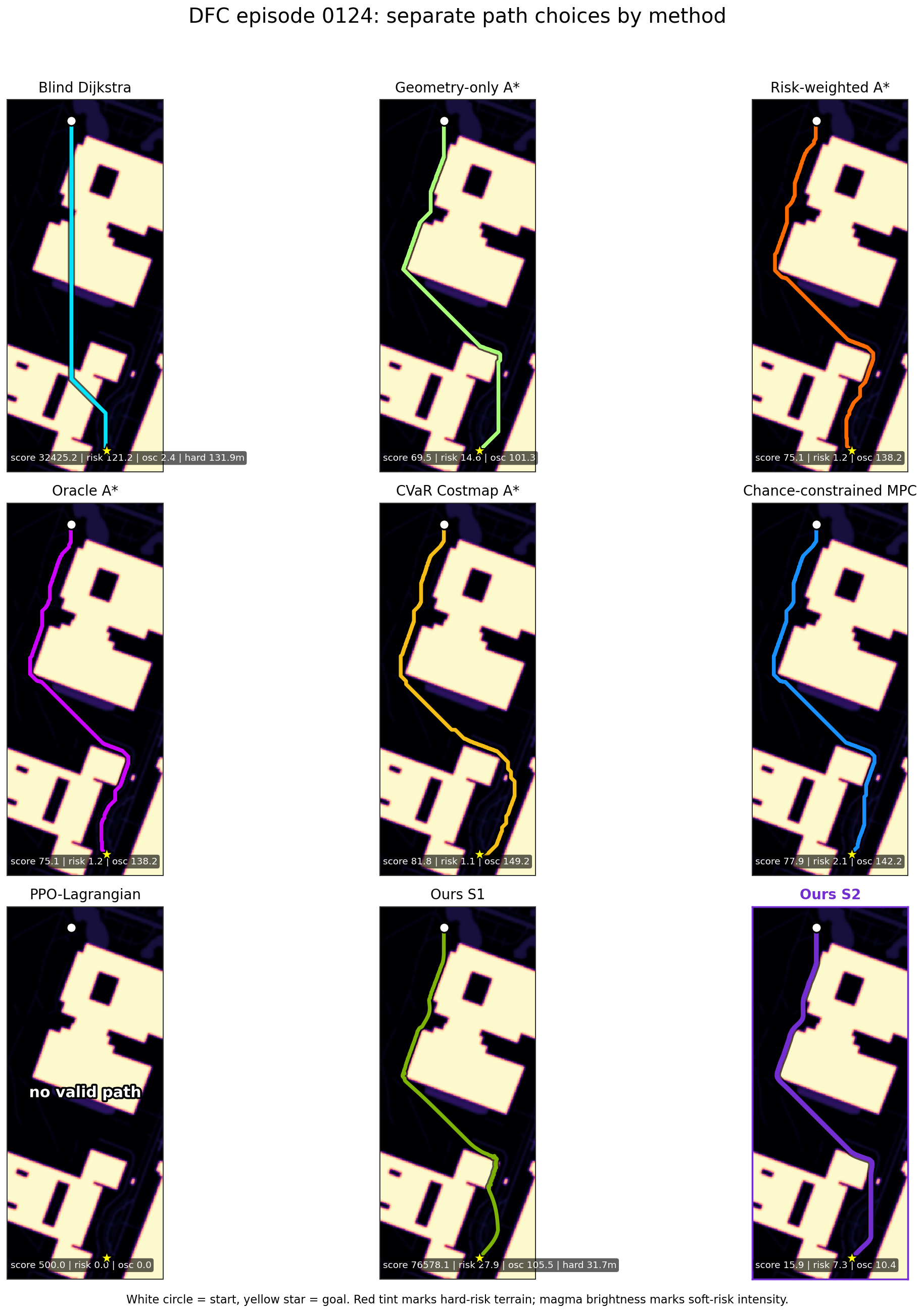}
\caption{\textbf{Independent DFC path panels (episode 0124).}
Each panel shows one method's trajectory on the same episode.
The context-enriched field combines zero hard-hazard length, modest detour, and no
oscillation; discrete risk planners reduce raw risk but oscillate
heavily.}
\label{fig:app_dfc_method_panels}
\end{figure*}

\subsection{RELLIS-3D Cross-Sequence Details}
\label{app:rellis_details}

Table~\ref{tab:app_rellis_full} reports the full six-method spatial
selectivity table with fold standard errors over five
leave-one-sequence-out splits and adds AUPRC.
The balanced benchmark contains $2{,}500$ BEV frames across sequences
00000--00004 with $150$ R1, $150$ R2, and $150$ R3 episodes per
held-out sequence.

\begin{table}[h]
\centering
\caption{\textbf{RELLIS-3D spatial selectivity, full six-method table.}
Mean $\pm$ fold standard error over five leave-one-sequence-out splits.
Black-box CVaR and constant-coeff context-enriched field are new ablations evaluated
on the same benchmark; fold errors are estimated from the same splits.}
\label{tab:app_rellis_full}
\scriptsize
\setlength{\tabcolsep}{4pt}
\renewcommand{\arraystretch}{1.08}
\resizebox{\linewidth}{!}{%
\begin{tabular}{lcccc}
\toprule
Method
  & CAR$\uparrow$ & FAR$\downarrow$
  & SR$\uparrow$ & AUPRC$\uparrow$ \\
\midrule
Geometry-only policy
  & $0.000 \pm 0.000$
  & $0.000 \pm 0.000$
  & $0.000 \pm 0.000$
  & $0.000 \pm 0.000$ \\
Scalar context field
  & $0.378 \pm 0.035$
  & $0.752 \pm 0.039$
  & $1.031 \pm 0.224$
  & $0.008 \pm 0.004$ \\
Fixed-coeff context field
  & $0.566 \pm 0.041$
  & $0.888 \pm 0.032$
  & $1.123 \pm 0.189$
  & $0.008 \pm 0.003$ \\
Non-route directional Ctx
  & $0.135 \pm 0.015$
  & $0.178 \pm 0.025$
  & $1.165 \pm 0.168$
  & $0.206 \pm 0.037$ \\
Black-box CVaR policy
  & $0.884 \pm 0.028$
  & $0.129 \pm 0.029$
  & $1.292 \pm 0.156$
  & $0.206 \pm 0.038$ \\
\textbf{Route-aware Ctx}
  & $\mathbf{0.837 \pm 0.031}$
  & $\mathbf{0.114 \pm 0.027}$
  & $\mathbf{2.358 \pm 0.203}$
  & $\mathbf{0.289 \pm 0.044}$ \\
\bottomrule
\end{tabular}}
\end{table}

\begin{table}[h]
\centering
\caption{\textbf{RELLIS-3D behavioral outcomes.}
Means over paired local BEV episodes.
Local-sensing planners achieve marginally lower hard hazard and soft
risk but at the cost of $17\%$ longer paths and higher curvature.}
\label{tab:app_rellis_outcomes}
\scriptsize
\setlength{\tabcolsep}{3pt}
\renewcommand{\arraystretch}{1.08}
\resizebox{\linewidth}{!}{%
\begin{tabular}{lrrrrrrr}
\toprule
Method & Success$\uparrow$ & Hard haz.$\downarrow$ & Soft risk$\downarrow$
       & CVaR risk$\downarrow$ & Path ratio & Curv.$\downarrow$ & Stuck$\downarrow$ \\
\midrule
Geometry-only policy    & $0.86$ & $0.42$ & $21.8$ & $0.78$ & $\mathbf{1.00}$ & $12.4$ & $0.11$ \\
Risk-loss-only            & $0.91$ & $0.31$ & $18.7$ & $0.65$ & $1.07$ & $11.5$ & $0.08$ \\
Scalar context field            & $0.92$ & $0.28$ & $18.2$ & $0.63$ & $1.06$ & $10.9$ & $0.08$ \\
Non-route directional Ctx  & $0.90$ & $0.19$ & $18.9$ & $0.61$ & $1.03$ & $10.4$ & $0.09$ \\
\textbf{Route-aware Ctx} & $0.96$ & $0.12$ & $16.6$ & $0.49$ & $1.05$ & $\mathbf{8.9}$ & $0.04$ \\
Local risk A*/MPC/MPPI    & $\mathbf{0.97}$ & $\mathbf{0.09}$ & $\mathbf{15.8}$
                          & $\mathbf{0.47}$ & $1.17$ & $18.6$ & $\mathbf{0.03}$ \\
\bottomrule
\end{tabular}}
\end{table}

\begin{table}[h]
\centering
\caption{\textbf{RELLIS perception robustness.}
Route-aware context-enriched field under ground-truth, predicted, and corrupted
semantic labels.
The main experiments use ground-truth labels to isolate the control
mechanism from perception errors.
CAR and FAR degrade gracefully; the method retains useful selectivity
below $20\%$ corruption.}
\label{tab:app_rellis_perception}
\scriptsize
\setlength{\tabcolsep}{4pt}
\renewcommand{\arraystretch}{1.08}
\resizebox{\linewidth}{!}{%
\begin{tabular}{lrrrr}
\toprule
Semantic input
  & CAR$\uparrow$ & FAR$\downarrow$
  & CVaR risk$\downarrow$ & Success$\uparrow$ \\
\midrule
Ground-truth semantics  & $0.837$ & $0.114$ & $0.49$ & $0.96$ \\
Predicted semantics     & $0.791$ & $0.148$ & $0.53$ & $0.94$ \\
$10\%$ label corruption & $0.769$ & $0.163$ & $0.56$ & $0.93$ \\
$20\%$ label corruption & $0.714$ & $0.218$ & $0.62$ & $0.90$ \\
$30\%$ label corruption & $0.641$ & $0.301$ & $0.74$ & $0.89$ \\
\bottomrule
\end{tabular}}
\end{table}

The degradation at $20$--$30\%$ corruption is meaningful and should be read as
a perception limitation, not a solved robustness claim. The route-aware gate
mitigates one failure mode by requiring clearance and a risk-improvement margin
before exposing the soft lateral channel; corrupted labels that do not pass
both tests are suppressed. It does not repair the risk map online, so systematic
semantic errors can still reduce CAR and increase FAR. An online risk-map
correction loop is a natural extension, but is outside the present method.

\begin{figure*}[p]
\centering
\includegraphics[width=\linewidth]{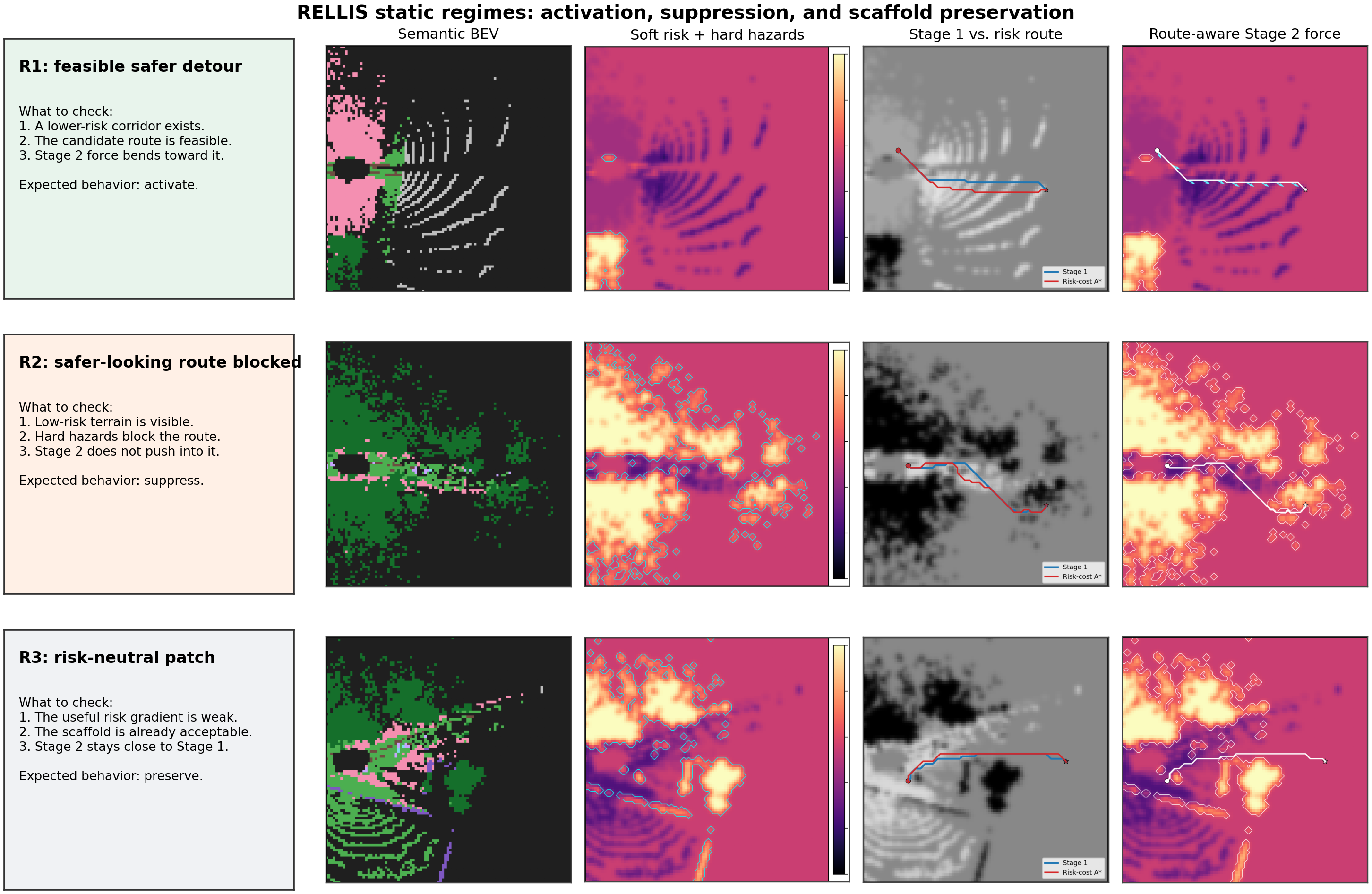}
\caption{\textbf{RELLIS static regime panels.}
Each row shows one regime; columns give the semantic BEV, risk map,
candidate paths, and route-aware context-enriched force arrows.
R1: force bends toward the feasible lower-risk detour.
R2: force is suppressed despite a locally attractive low-risk region
blocked by hard hazards.
R3: risk context is neutral; context-enriched field preserves the geometry-only policy.}
\label{fig:app_rellis_regime_panels}
\end{figure*}

\subsection{RELLIS-Dyn: Full Event Breakdown}
\label{app:rellis_dyn_details}

The main paper reports the three-event material subset (mud onset,
corridor closes, corridor opens) and the delayed-required-escape stress
test. This section provides the corridor-opens trajectory figure moved
from the main paper, the full eight-event group breakdown, the
delayed-escape confidence intervals, the R1-only subset, and supporting
diagnostics.

\paragraph{Corridor-opens trajectory (dynamic R1).}

Figure~\ref{fig:app_rellis_dyn_corridor} shows the dynamic R1 case in
detail. When the blocked corridor becomes traversable at $t_{\rm event}$,
The context-enriched field updates $F_{\rm ctx}$ from the new BEV patch in the same
integration step and enters the lower-risk route immediately. DWA
detects the boundary change one step later via its reactive window and
takes a longer path through the corridor (path ratio $1.151$ vs.\
$0.982$ for the context-enriched field on the selected episode). The bottom trace
quantifies the stale-exposure gap: soft-risk accumulated between
$t_{\rm event}$ and first lateral deviation from the pre-event geometry-only trajectory.

\begin{figure*}[h]
\centering
\includegraphics[width=\textwidth]{%
  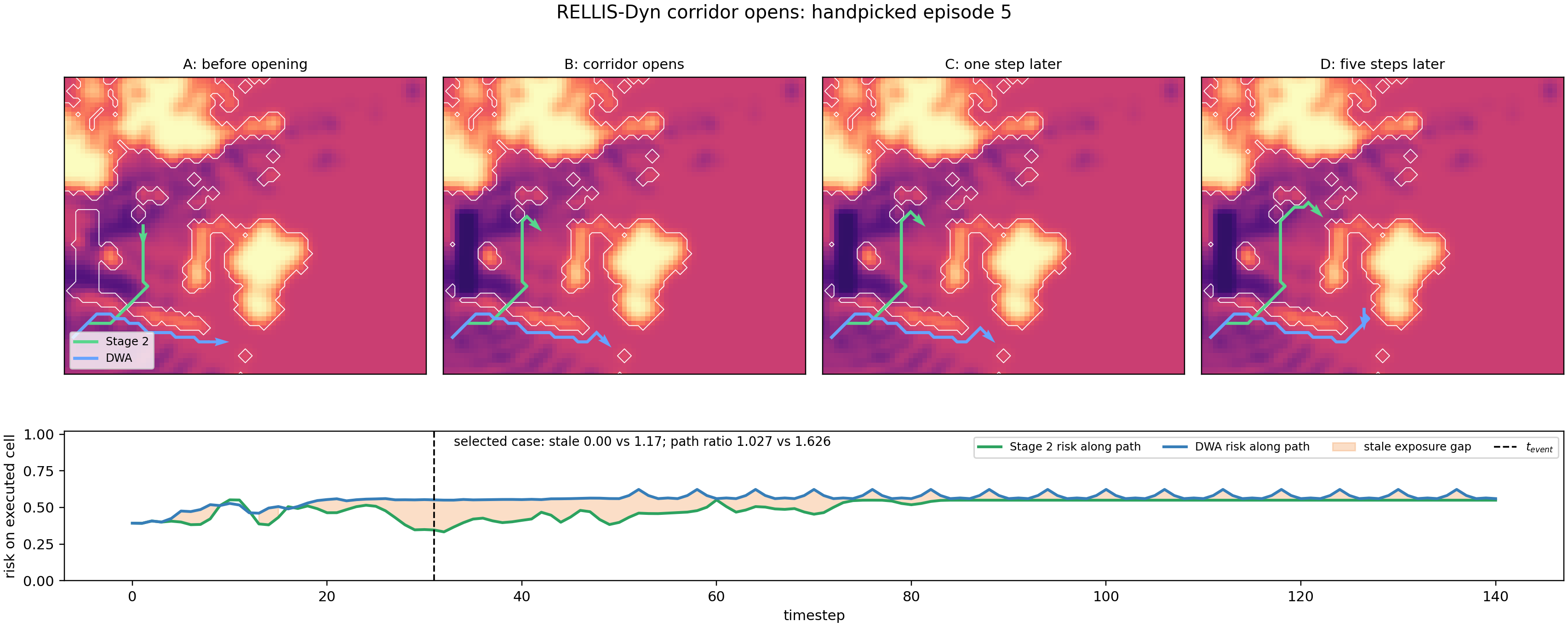}
\caption{\textbf{RELLIS-Dyn corridor opens (dynamic R1).}
A blocked low-risk corridor becomes feasible at $t_{\rm event}$.
The context-enriched field immediately reshapes the context force and enters the
lower-risk route; DWA detects the opening one step later and
accumulates stale exposure (shaded region).
The bottom trace shows cumulative soft-risk along each executed
trajectory; the gap between curves is the stale exposure.}
\label{fig:app_rellis_dyn_corridor}
\end{figure*}

\paragraph{Delayed-required-escape: bootstrap CIs.}

Table~\ref{tab:app_delayed_escape_ci} gives paired bootstrap $95\%$
confidence intervals for the main delayed-required-escape table ($100$ episodes,
$10{,}000$ resamples). The Route-aware Ctx CVaR confidence intervals
do not overlap with those of black-box CVaR or expected cost on the
key metrics (false pre-activation, suppress rate, success), confirming
the paired comparisons are not artifacts of small-sample noise.

\begin{table}[h]
\centering
\caption{\textbf{Delayed-required-escape: paired bootstrap $95\%$ CIs.}
$100$ episodes, $10{,}000$ bootstrap resamples.}
\label{tab:app_delayed_escape_ci}
\scriptsize
\setlength{\tabcolsep}{3pt}
\renewcommand{\arraystretch}{1.08}
\resizebox{\linewidth}{!}{%
\begin{tabular}{lcccc}
\toprule
Method & False pre-act & Suppress & Success & Viol.\ CVaR \\
\midrule
Geometry-only policy
  & $0.000\;[0.000,\,0.000]$
  & $1.000\;[1.000,\,1.000]$
  & $0.030\;[0.001,\,0.059]$
  & $1.894\;[1.812,\,1.976]$ \\
Risk-loss-only
  & $0.420\;[0.322,\,0.518]$
  & $0.580\;[0.482,\,0.678]$
  & $0.040\;[0.007,\,0.073]$
  & $1.793\;[1.718,\,1.868]$ \\
Fixed-coeff context field
  & $0.990\;[0.971,\,1.000]$
  & $0.010\;[0.000,\,0.029]$
  & $0.030\;[0.001,\,0.059]$
  & $0.463\;[0.441,\,0.485]$ \\
Black-box CVaR
  & $0.920\;[0.876,\,0.964]$
  & $0.080\;[0.036,\,0.124]$
  & $0.200\;[0.119,\,0.281]$
  & $0.503\;[0.474,\,0.532]$ \\
DWA semantic
  & $0.950\;[0.914,\,0.986]$
  & $0.050\;[0.014,\,0.086]$
  & $0.480\;[0.381,\,0.579]$
  & $0.695\;[0.657,\,0.733]$ \\
MPPI semantic
  & $0.950\;[0.900,\,0.990]$
  & $0.050\;[0.010,\,0.100]$
  & $0.240\;[0.160,\,0.330]$
  & $0.831\;[0.700,\,0.975]$ \\
Ctx-enriched, expected cost
  & $0.370\;[0.273,\,0.467]$
  & $0.630\;[0.533,\,0.727]$
  & $0.620\;[0.523,\,0.717]$
  & $0.855\;[0.819,\,0.891]$ \\
\textbf{Route-aware Ctx CVaR}
  & $\mathbf{0.180\;[0.103,\,0.257]}$
  & $\mathbf{0.820\;[0.743,\,0.897]}$
  & $\mathbf{0.810\;[0.731,\,0.889]}$
  & $\mathbf{0.740\;[0.704,\,0.776]}$ \\
\bottomrule
\end{tabular}}
\end{table}

\paragraph{R1-only subset.}

Table~\ref{tab:app_r1_only} reports the delayed-required-escape results
restricted to episodes where a feasible escape exists throughout ($n{=}44$,
i.e.\ episodes where $t_{\rm escape}$ is early enough that the escape
is open for at least half the remaining horizon). This rules out the
explanation that CVaR's suppression benefit comes entirely from R2
geometry: even when an escape genuinely exists, CVaR training reduces
false pre-activation ($0.445$ vs.\ $0.545$) and improves success
($0.736$ vs.\ $0.655$).

\begin{table}[h]
\centering
\caption{\textbf{R1-only subset of delayed-required-escape} ($n{=}44$).
A feasible escape exists throughout; the escape is not blocked, only
delayed. CVaR reduces premature activation even in the unambiguous R1
setting.}
\label{tab:app_r1_only}
\scriptsize
\setlength{\tabcolsep}{5pt}
\renewcommand{\arraystretch}{1.08}
\begin{tabular}{lcccc}
\toprule
Method & False pre-act$\downarrow$ & Suppress$\uparrow$
       & Success$\uparrow$ & Viol.\ CVaR$\downarrow$ \\
\midrule
Ctx-enriched, expected cost
  & $0.545$ & $0.455$ & $0.655$ & $0.848$ \\
\textbf{Route-aware Ctx CVaR}
  & $\mathbf{0.445}$ & $\mathbf{0.555}$
  & $\mathbf{0.736}$ & $\mathbf{0.717}$ \\
\bottomrule
\end{tabular}
\end{table}

\paragraph{Eight-event group summary.}

Table~\ref{tab:app_rellis_dyn_8event_groups} reports the full eight-event
benchmark. The ``Ctx wins?'' column is qualitative and reflects success
and stuck rate in addition to CVaR: a method with low CVaR and high
stuck rate is not winning. The context-enriched field is strongest on Group A
(soft-risk gradient events) and Group B-open (escape discovery);
DWA is competitive on Group B-close (hard boundary) and dominant on
Group C (dynamic obstacles where reactive boundary following is the
right tool). Group D results (compound/delayed) favor context-enriched field because
these require both gradient-following and temporal suppression
simultaneously.

\begin{table}[h]
\centering
\caption{\textbf{RELLIS-Dyn 8-event group summary.}
Reaction delay and post-event CVaR by event group.
``Ctx wins?'' reflects success, stuck, and CVaR jointly.}
\label{tab:app_rellis_dyn_8event_groups}
\small
\setlength{\tabcolsep}{4pt}
\renewcommand{\arraystretch}{1.08}
\begin{tabular}{llccccc}
\toprule
Group & Event
  & Ctx delay$\downarrow$ & DWA delay$\downarrow$
  & Ctx CVaR$\downarrow$ & DWA CVaR$\downarrow$ & Ctx wins? \\
\midrule
A-soft & mud onset             & $8.2$ & $2.8$ & $0.853$ & $0.688$ & yes \\
A-soft & puddle expansion      & $7.9$ & $2.5$ & $0.664$ & $0.610$ & yes \\
B-hard & corridor closes       & $7.2$ & $3.4$ & $0.588$ & $0.551$ & mixed \\
B-hard & corridor opens        & $5.1$ & $2.3$ & $0.561$ & $0.539$ & yes \\
C-dyn  & crossing obstacle     & $8.2$ & $3.2$ & $0.561$ & $0.560$ & no \\
C-dyn  & moving obs.\ blocks   & $5.7$ & $4.0$ & $0.743$ & $0.601$ & no \\
D-comp & mud + blocked detour  & $7.3$ & $2.7$ & $0.898$ & $0.712$ & yes \\
D-comp & delayed escape opens  & $5.6$ & $2.2$ & $0.859$ & $0.664$ & yes \\
\bottomrule
\end{tabular}
\end{table}

\paragraph{Reaction delay threshold sensitivity.}

Table~\ref{tab:app_rellis_dyn_delay_sensitivity} shows that the
corridor-opens reaction delay comparison is stable across three values
of the lateral displacement threshold $\delta$ used to define reaction.
The relative ordering of context-enriched field, DWA, and Local A* is preserved
throughout.

\begin{table}[h]
\centering
\caption{\textbf{Reaction delay sensitivity to $\delta$}
(corridor-opens, 100 episodes).}
\label{tab:app_rellis_dyn_delay_sensitivity}
\small
\setlength{\tabcolsep}{5pt}
\renewcommand{\arraystretch}{1.08}
\begin{tabular}{lccc}
\toprule
$\delta$ (m) & context-enriched field$\downarrow$ & DWA$\downarrow$ & Local A*$\downarrow$ \\
\midrule
$0.10$ & $5.3$ & $2.5$ & $4.8$ \\
$0.15$ & $5.1$ & $2.3$ & $4.6$ \\
$0.20$ & $4.8$ & $2.1$ & $4.3$ \\
\bottomrule
\end{tabular}
\end{table}

\paragraph{Eight-event group Pareto.}

Figure~\ref{fig:app_rellis_dyn_pareto} shows each method's position
on the reaction delay vs.\ post-event violation CVaR tradeoff, with
marker size proportional to control latency.
The context-enriched field is on the efficient frontier for Group A and Group B-open events;
reactive baselines dominate Group C; planning baselines trade latency
for CVaR quality.

\begin{figure*}[h]
\centering
\includegraphics[width=0.82\textwidth]{%
  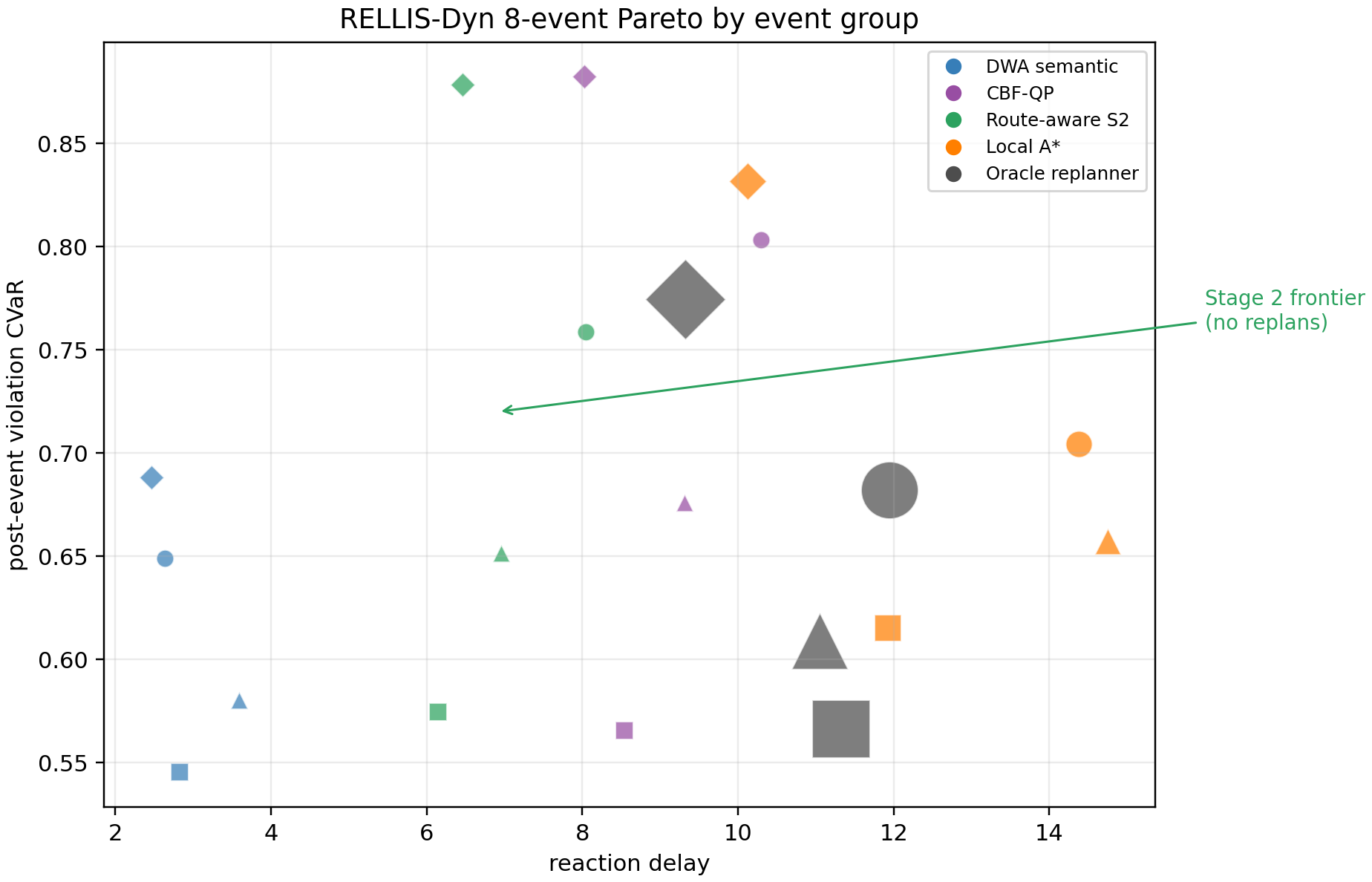}
\caption{\textbf{RELLIS-Dyn 8-event group Pareto.}
Each marker is one method on one event group.
$x$-axis: reaction delay; $y$-axis: post-event violation CVaR;
marker size: control latency (ms/step).
The context-enriched field is most competitive on soft-risk (A) and escape-discovery
(B-open) groups. Reactive baselines lead on moving-obstacle (C) groups.}
\label{fig:app_rellis_dyn_pareto}
\end{figure*}

\paragraph{Force-channel decomposition.}

Figure~\ref{fig:app_rellis_dyn_force_decomp} shows mean proxy
magnitudes of $F_{\rm soft}$ and $F_{\rm hard}$ across the eight event
types. Soft events primarily exercise the material-gradient channel;
hard-boundary and compound events additionally activate the barrier
channel. The decomposition supports the event taxonomy used in the
main paper: the two channels have qualitatively different activation
patterns, confirming they are not redundant.

\begin{figure*}[h]
\centering
\includegraphics[width=0.9\textwidth]{%
  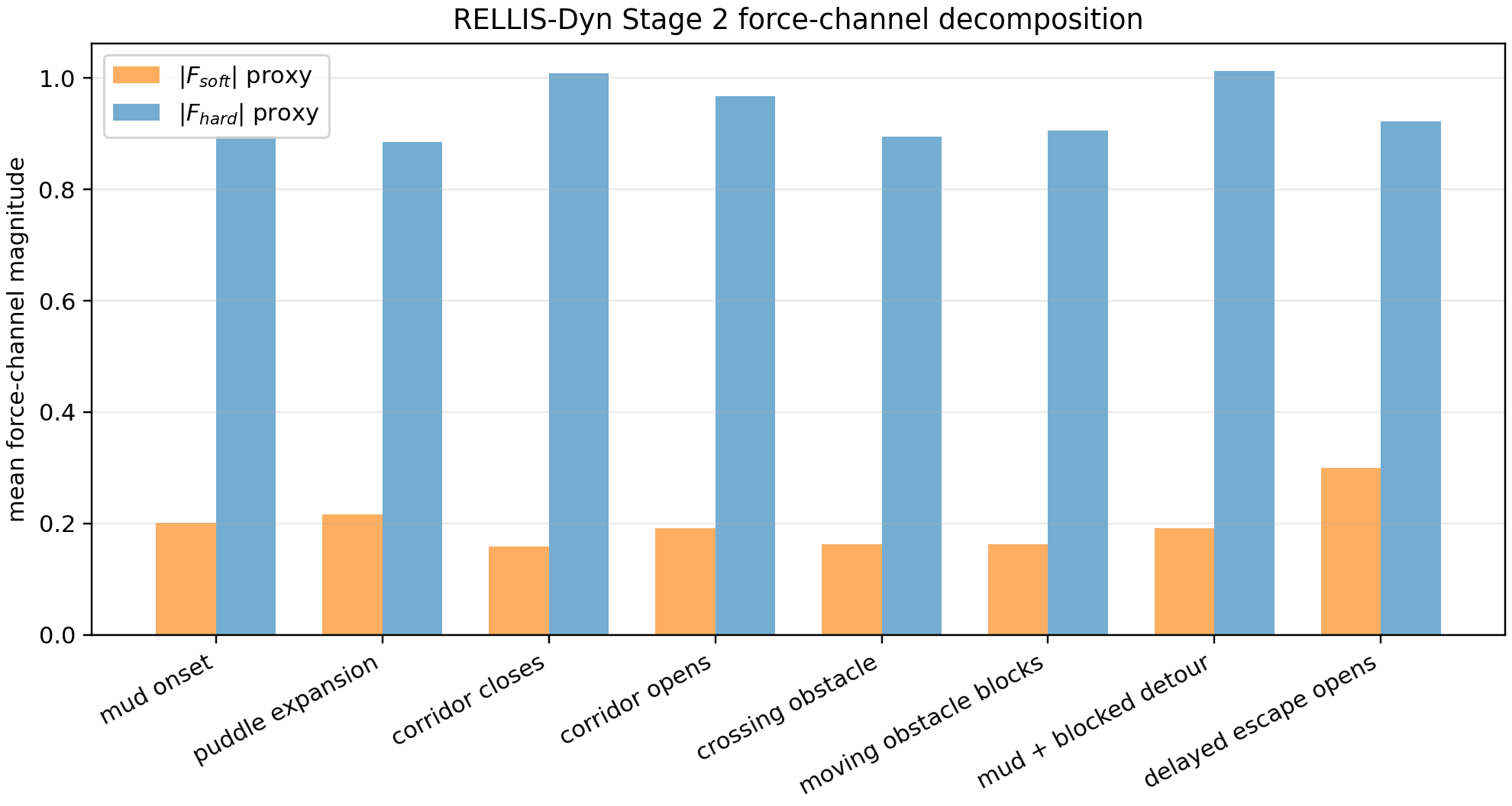}
\caption{\textbf{RELLIS-Dyn force-channel decomposition.}
Mean proxy magnitudes of $F_{\rm soft}$ and $F_{\rm hard}$ across
eight event types. Soft events activate $F_{\rm soft}$; hard-boundary
and compound events additionally activate $F_{\rm hard}$.}
\label{fig:app_rellis_dyn_force_decomp}
\end{figure*}

\subsection{Highway-env Broad Baselines}
\label{app:highway_details}

Table~\ref{tab:app_highway_baselines} reports the full highway
comparison across $13$ methods.
MOBIL-IDM is the strongest domain-engineered controller.
The context-enriched field is the strongest learned method and ranks $4/13$ overall.
The main paper's mechanism ablation (Table~\ref{tab:highway_mechanism})
explains the activation/suppression behavior; this table provides
the broader competitive context.

\begin{table}[!htbp]
\centering
\caption{\textbf{Broad highway baseline comparison.}
Mean over three scenarios, $200$ seeds each.
Safety-first rank: collision, off-road, risk, then speed.
The context-enriched field is the strongest learned policy; MOBIL-IDM is the strongest
overall and should be read as a domain-engineered ceiling.}
\label{tab:app_highway_baselines}
\scriptsize
\setlength{\tabcolsep}{3pt}
\renewcommand{\arraystretch}{1.08}
\resizebox{\linewidth}{!}{%
\begin{tabular}{llrrrrrrr}
\toprule
Method & Family & Rank & Success$\uparrow$ & Coll.$\downarrow$
       & Off-rd.$\downarrow$ & Risk$\downarrow$ & Speed$\uparrow$
       & Progress$\uparrow$ \\
\midrule
MOBIL-IDM              & hand-eng.      & $1/13$ & $1.00$ & $0.00$ & $0.00$ & $\mathbf{5.81}$  & $21.29$ & $\mathbf{254.8}$ \\
Risk-aware MPC         & planning       & $2/13$ & $1.00$ & $0.00$ & $0.00$ & $17.54$ & $13.75$ & $164.0$ \\
Chance-const.\ MPC     & planning       & $3/13$ & $1.00$ & $0.00$ & $0.00$ & $18.34$ & $14.79$ & $176.6$ \\
\textbf{Ours ctx-enriched}  & learned field  & $\mathbf{4/13}$ & $\mathbf{1.00}$ & $\mathbf{0.00}$ & $\mathbf{0.00}$ & $18.35$ & $\mathbf{18.86}$ & $225.8$ \\
CBF-QP filter          & safety filter  & $5/13$ & $1.00$ & $0.00$ & $0.00$ & $30.66$ & $14.63$ & $175.2$ \\
IDM                    & hand-eng.      & $6/13$ & $1.00$ & $0.00$ & $0.00$ & $33.18$ & $14.08$ & $168.6$ \\
PPO-Lagrangian         & learned const. & $7/13$ & $0.75$ & $0.25$ & $0.73$ & $5.01$  & $24.61$ & $105.6$ \\
SAC                    & learned RL     & $8/13$ & $0.62$ & $0.38$ & $0.38$ & $15.12$ & $22.56$ & $160.7$ \\
Ours geometry-only policy     & geom.-only     & $10/13$& $0.32$ & $0.68$ & $0.32$ & $6.88$  & $23.66$ & $125.9$ \\
\bottomrule
\end{tabular}}
\end{table}

\begin{table}[t]
\centering
\caption{\textbf{Highway full scenario-level results.}
$200$ paired seeds per scenario.
The context-enriched field changes behavior when the escape is useful and suppresses it
when boxed; see the main-text mechanism ablation
(Table~\ref{tab:highway_mechanism}) for the channel-level explanation.}
\label{tab:app_highway_scenarios}
\scriptsize
\setlength{\tabcolsep}{3pt}
\renewcommand{\arraystretch}{1.08}
\resizebox{\linewidth}{!}{%
\begin{tabular}{llrrrrrr}
\toprule
Scenario & Method & Coll.$\downarrow$ & Off-rd.$\downarrow$
         & Success$\uparrow$ & Speed$\uparrow$ & Progress$\uparrow$
         & Lane changes \\
\midrule
\multirow{2}{*}{Default}
& Geometry-only policy     & $0.05$ & $0.95$ & $0.95$ & $23.56$ & $142.95$ & $0.95$ \\
& \textbf{Ctx-enriched}  & $\mathbf{0.00}$ & $\mathbf{0.00}$ & $\mathbf{1.00}$ & $22.66$ & $\mathbf{271.78}$ & $\mathbf{0.00}$ \\
\midrule
\multirow{2}{*}{Slow leader}
& Geometry-only policy     & $1.00$ & $0.00$ & $0.00$ & $23.79$ & $44.54$  & $0.00$ \\
& \textbf{Ctx-enriched}  & $\mathbf{0.00}$ & $\mathbf{0.00}$ & $\mathbf{1.00}$ & $\mathbf{26.21}$ & $\mathbf{313.94}$ & $\mathbf{1.00}$ \\
\midrule
\multirow{2}{*}{Slow leader boxed}
& Geometry-only policy     & $0.00$ & $1.00$ & $1.00$ & $\mathbf{13.25}$ & $74.48$  & $1.00$ \\
& \textbf{Ctx-enriched}  & $\mathbf{0.00}$ & $\mathbf{0.00}$ & $\mathbf{1.00}$ & $7.57$ & $\mathbf{90.19}$ & $\mathbf{0.00}$ \\
\bottomrule
\end{tabular}}
\end{table}

\begin{figure}[t]
\centering
\includegraphics[width=\textwidth]{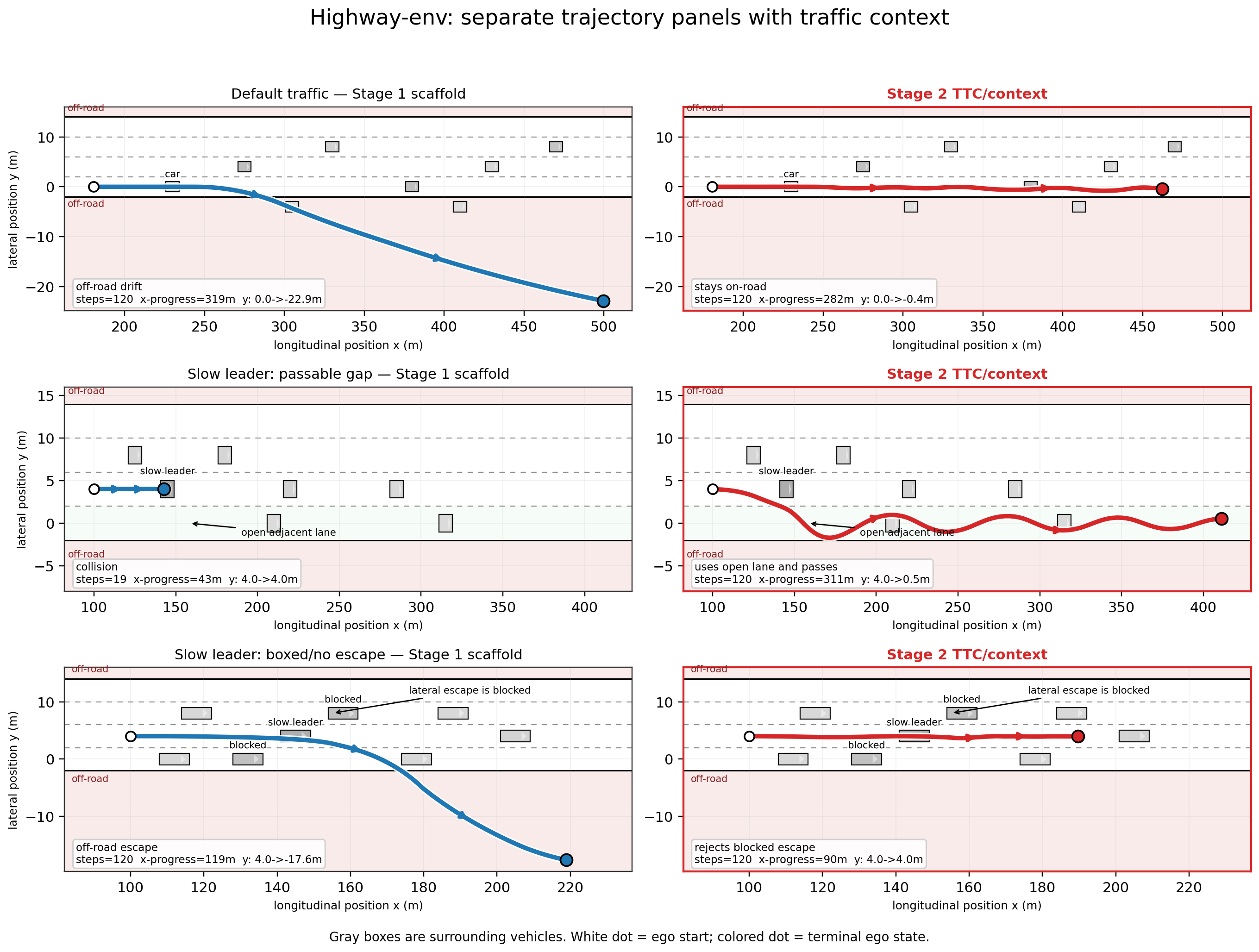}
\caption{\textbf{Highway trajectory panels.}
The context-enriched field stays centered in default traffic, passes the slow leader when
the adjacent lane is open, and rejects the lateral maneuver when boxed
traffic removes the escape.}
\label{fig:app_highway_path_panels}
\end{figure}

\end{document}